\documentclass[10pt,letterpaper,twocolumn]{article}
\usepackage[T1]{fontenc}
\usepackage[textwidth=7.1in,textheight=9.25in,columnsep=.30in]{geometry}
\usepackage[hyphens]{url}
\usepackage{xurl}
\usepackage{graphicx,booktabs,amsmath,amssymb,amsthm,xcolor}
\usepackage{newtxtext,newtxmath}
\usepackage{microtype}
\usepackage[round,authoryear]{natbib}
\usepackage[font=small,labelfont=bf]{caption}
\usepackage[hidelinks]{hyperref}
\hypersetup{pdftitle={Learned, Relied Upon, or Necessary? Separating Checkpoint Dependence from Task-Level Value in Sheaf Graph Neural Networks},pdfauthor={Yi Liu}}
\newcommand{\resultstat}[2]{\ensuremath{#1_{\textcolor[gray]{.38}{\scriptscriptstyle #2}}}}
\renewenvironment{abstract}
  {\par\small\begin{center}\bfseries\abstractname\end{center}\noindent\ignorespaces}
  {\par}
\newtheorem{theorem}{Theorem}

\newtheorem{corollary}{Corollary}
\DeclareMathOperator{\Alg}{Alg}

\title{Learned, Relied Upon, or Necessary? Separating Checkpoint Dependence from Task-Level Value in Sheaf Graph Neural Networks}
\author{Yi Liu\\University of Science and Technology of China\\
\href{mailto:scnuliuyi@mail.ustc.edu.cn}{\texttt{scnuliuyi@mail.ustc.edu.cn}}}
\date{}

\begin{document}
\maketitle
\begin{abstract}
Sheaf neural networks learn edge-conditioned restriction maps so a node transforms features differently for each neighbor, yet fixing every map to the identity remains competitive on heterophilous benchmarks. We hypothesize that a trained network relies heavily on learned maps at a fixed checkpoint even when simpler map rules match or beat it after retraining. A fixed linear readout illustrates how maps can change while predictions stay fixed. In a rotated-frame Gaussian model, refitting the readout recovers class signal after map removal. These examples motivate separate tests of what training updates (Learned), what a fixed checkpoint depends on (Relied Upon), and whether native edge-conditioned full matrices outperform the tested retrained alternatives (Necessary). To test these three properties on real graphs, we introduce Learned--Relied--Necessary Evaluation (LRN-Evaluation), pairing map-weight movement tracking and frozen-checkpoint interventions with matched retraining across seven fixed, capacity-matched, layer-shared, and reassigned controls. Across five heterophilous node-classification benchmarks, shuffling maps lowers test scores at every full Deep Neural Sheaf Diffusion checkpoint. After matched retraining, reshuffling maps each forward pass gives mean scores within $.003$ of Full on three tasks. On Minesweeper, sharing one learned map pair per layer raises AUROC from $.9347$ to $.9482$, while on Roman-Empire the full model beats all seven controls. We recommend testing all three stages on the target task before choosing between shared and edge-conditioned maps.
\end{abstract}
\begin{center}\small
\textbf{Code}\quad\url{https://github.com/daoisdao/LRN-Evaluation}
\end{center}

\section{Introduction}

Sheaf neural networks (SNNs) are graph neural networks (GNNs) that learn a task-specific geometry for message passing. Each node carries a local feature space called a node stalk, and each edge carries two linear restriction maps, one per endpoint, that project incident node representations into a shared edge stalk \citep{hansen2020sheaf,bodnar2022neural}. Standard SNNs predict these edge-conditioned restriction maps from the two endpoint node embeddings, letting a node transform messages differently for each neighbor instead of averaging uniformly. That flexibility matters on heterophilous graphs, where connected nodes frequently hold different labels and uniform averaging blurs class boundaries. Neural Sheaf Diffusion (NSD) shows that the choice of restriction maps determines which classes diffusion separates in the limit, and learned-map NSD models achieve strong accuracy on heterophilous benchmarks \citep{bodnar2022neural}.

Results with the Identity Sheaf Network (ISN) challenge how much of this empirical gain comes from learning restriction maps. Across five heterophilous benchmarks, fixing every map to the identity throughout training matches several learned-map SNNs \citep{hernandez2026necessity}. Training from scratch with identity maps allows hidden layers and node readouts to adapt directly to unaligned messages. Joint training with edge-conditioned maps instead couples the rest of the network to the transformations those maps supply. We investigate which properties of learned maps a trained SNN depends on at inference, and which still improve prediction once the full network trains from scratch under a simpler map rule.

Two analytical examples show how this gap can arise. Under a linear readout applied to mean-aggregated transformed messages, each incoming message enters along one effective direction, so map perturbations orthogonal to the readout weights leave every score unchanged. Conversely, a perturbation that hurts a fixed readout does not necessarily erase class information from the aggregated messages. In a Gaussian model where neighbors observe a shared class signal in rotated local frames, replacing alignment maps with identity rotates the aggregated mean away from the fitted readout. A refitted readout realigns with the surviving mean, and only complete signal cancellation drives the optimal identity-map readout to chance. Together these cases separate what optimization updates (Learned), what a fixed checkpoint depends on (Relied Upon), and whether the full model outperforms every specified alternative trained from scratch under a shared protocol (Necessary).

\begin{figure*}[t]
\centering
\includegraphics[width=\textwidth]{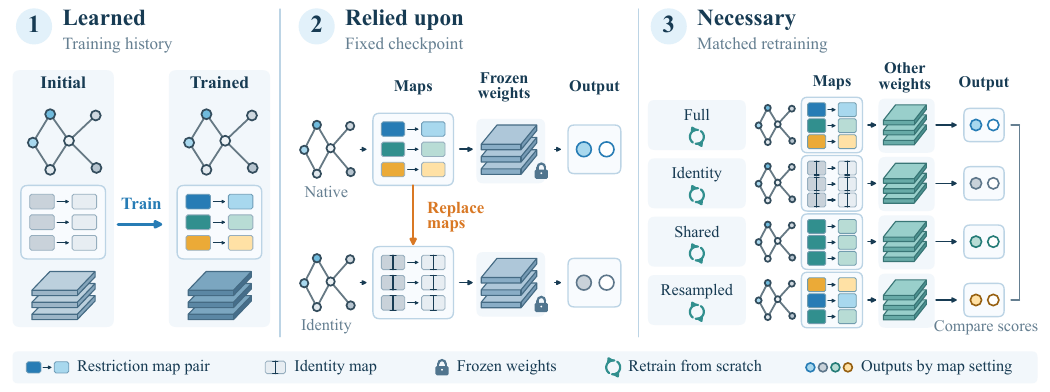}
\caption{The three stages of LRN-Evaluation: Learned, Relied Upon, and Necessary. In the rightmost panel, maps satisfy the Necessary criterion when Full outperforms every retrained control in the suite. Training history tracks parameter movement in the map predictor. Checkpoint interventions freeze all weights while replacing or shuffling maps. Matched retraining tests how much of Full's predictive gain survives when each control trains from scratch.}
\label{fig:claims}
\end{figure*}

To test these three properties in SNNs, we introduce Learned--Relied--Necessary Evaluation (LRN-Evaluation, Figure~\ref{fig:claims}). Starting from the full edge-conditioned model (Full), LRN-Evaluation executes three stages. First, it tracks how far the map predictor's weights move from initialization during training. Second, it runs checkpoint interventions that freeze all weights while replacing predicted maps with identities, layer-wide means, or cross-edge shuffles at test time. Third, it runs matched retraining across Full and seven simpler or cross-edge reassigned controls: identity, diagonal, parameter-matched, layer-shared, fixed-shuffle, and freshly resampled maps. Each model trains from scratch on identical splits, training budgets, and validation rules so the whole network adapts to its own message geometry.

We apply LRN-Evaluation to Deep Neural Sheaf Diffusion (DNSD) \citep{bourgerie2026deep} across the five heterophilous node-classification benchmarks of \citet{platonov2023critical}. Shuffling restriction maps across edges at test time lowers the score at all 50 Full checkpoints (ten official splits per dataset), yet matched retraining reverses map rankings across tasks. On Minesweeper, a regular grid graph, substituting identity maps at a frozen Full checkpoint drops mean AUROC by $.371$, yet retraining from scratch with identity maps surpasses Full. Sharing a single learned map pair per layer lifts mean AUROC by $.0135$ over Full using under $1\%$ of its map parameters. On Roman-Empire, a word-dependency graph labeled by syntactic roles, Full beats its closest retrained alternative by 2.87 percentage points. Composing each edge's two restriction maps gives a sender-to-receiver edge transport map. On Roman-Empire, ordered endpoint role pairs explain about half the variance of these transport maps in every layer. Because shared and edge-conditioned maps reverse order under one architecture and protocol despite consistent checkpoint sensitivity, model choice on these tasks depends on the retrained comparison.

Our contributions are threefold.
\begin{itemize}
\item We introduce LRN-Evaluation for SNNs. The framework pairs fixed-checkpoint interventions (identity, layer-mean, and shuffled maps) with matched retraining controls for parameter capacity, layer-wide sharing, and edge-specific assignment.
\item We separate map-predictor weight movement, checkpoint reliance, and post-retraining value. A fixed linear readout and a rotated-frame Gaussian model illustrate why these measurements can differ.
\item We benchmark eight DNSD map designs on five heterophilous graphs and run targeted transfer controls on NSD and Directed Sheaf Neural Networks (DSNN). Native edge-conditioned DNSD beats all seven retrained controls on Roman-Empire, and layer-shared DNSD beats Full on Minesweeper.
\end{itemize}

\section{Related Work}

\paragraph{Learning a geometry for message passing.}
A cellular sheaf equips a graph with restriction maps that compare incident node-stalk signals across edges \citep{hansen2019spectral,hansen2020sheaf}. When smooth node signals are directly observed, one can fit a sheaf Laplacian for neighbor consistency \citep{hansen2019smooth} or estimate restriction maps from the signals \citep{dinino2025restriction}. Known local geometry can also supply tangent-frame transport directly \citep{battiloro2024tangent}, and synthetic diffusion tasks probe specific geometric priors \citep{caralt2024joint}. In semi-supervised node classification, however, only node labels supervise training and distinct maps can produce identical predictions. Our rotated-frame model supplies known ground-truth alignment so we compute exactly how much class signal survives under identity maps, motivating post-retraining evaluation on real graphs.

\paragraph{Restriction maps and propagation rules.}
NSD characterizes how diagonal, orthogonal, and general restriction-map families govern asymptotic class separation \citep{bodnar2022neural}. Later SNNs explore flat vector bundles \citep{bamberger2025bundle}, symmetric positive definite sheaves \citep{peng2026spd}, and alternative propagation operators \citep{zaghen2024nonlinear,ribeiro2026cooperative} based on polynomial filters \citep{borgi2026polynsd}, adjacency updates \citep{bourgerie2026deep}, or directed complex-valued diffusion \citep{fiorini2026reloaded}. Bayesian and nonlinear sheaf networks also vary the map learner \citep{gillespie2024bayesian}, while spectral objectives and quiver analysis study the propagation of node representations \citep{choi2026pacbayes,donmez2026degeneracy}. Cross-architecture comparisons conflate map parameterization with propagation operators, a distinction also noted in inductive sheaf benchmarks \citep{fiorini2026inductive}. We hold the DNSD propagation rule fixed across map families and test selected controls under other propagation rules.

\paragraph{From interpreting a trained map to choosing its replacement.}
ISN evaluates map necessity by retraining with identity maps on every edge \citep{hernandez2026necessity}. When identity maps lose accuracy, three mechanisms remain confounded: a single layer-shared transformation, a diverse multiset of transformations, or assigning each map specifically to its generating edge. On triangle counting, \citet{grover2026holonomy} test trained connection transports under shuffling, flattening, and identity replacement alongside a retrained identity baseline, and \citet{guadagnuolo2026bochner} perturb learned sheaf Laplacians to identify influential nodes. More broadly, standard GNNs rely on input graph edges even on tasks where ignoring graph structure improves generalization \citep{bechler2024graphs}, echoing attention faithfulness \citep{wiegreffe2019attention}, feature-removal retraining \citep{hooker2019benchmark}, and Rashomon-set variable importance \citep{fisher2019all}. Unlike studies that only retrain an identity baseline or only perturb a fixed checkpoint, LRN-Evaluation combines fixed-checkpoint interventions with matched retraining controls for parameter capacity, layer-wide sharing, and edge-specific assignment (Table~\ref{tab:isn-scope} in Appendix~\ref{app:isn}).

\section{LRN-Evaluation for Sheaf Neural Networks}
\label{sec:lrn}

In SNNs, every undirected edge $e=\{u,v\}$ carries two restriction maps $F_{u\unlhd e}, F_{v\unlhd e}\in\mathbb R^{d\times d}$ that project $d$-dimensional endpoint node stalks into a shared $d$-dimensional edge stalk. In NSD and DNSD, the induced edge transport map from sender $u$ to receiver $v$ is
\begin{equation}
 A_{u\to v}=F_{v\unlhd e}^{\top}F_{u\unlhd e},
 \label{eq:action}
\end{equation}
which transports the sender's features before normalization and aggregation. An orthogonal basis change on the edge stalk rotates both restriction maps without changing $A_{u\to v}$. We therefore keep each source--target pair $(F_{u\unlhd e}, F_{v\unlhd e})$ intact when shuffling and quantify cross-edge variation through $A_{u\to v}$.

\subsection{Training History and Checkpoint Reliance}

The first two stages of LRN-Evaluation (Figure~\ref{fig:claims}) inspect a single Full run. The Learned stage records the root-mean-square (RMS) parameter change $M$ of the map predictor from initialization to the selected checkpoint. The checkpoint intervention measures how predictions depend on these maps.

The Relied Upon stage freezes all weights and compares held-out scores before and after intervening on restriction maps in the forward pass:
\begin{equation}
 R_\iota=S(\widehat T,\widehat w)-S(\iota(\widehat T),\widehat w),
 \label{eq:reliance}
\end{equation}
where $S$ is the held-out metric, $\widehat T$ collects the fitted map predictors, $\widehat w$ collects all remaining weights, and $\iota$ alters each layer's predicted maps before propagation, with $R_\iota>0$ indicating degradation.

We apply three checkpoint interventions. First, replacing every map with the identity tests reliance on learned transport overall. Second, replacing source and target maps with their layer-wide means gives every edge one shared pair. Third, cross-edge shuffling preserves each layer's exact multiset of predicted pairs while permuting which edge receives each pair, testing sensitivity to edge assignment at that checkpoint. Because layers compose, early-layer interventions also shift downstream representations and map predictions.

\subsection{Matched Retraining of Alternative Map Designs}

Matched retraining trains Full and each alternative map family independently from scratch on identical splits, seeds, budgets, and validation rules so each model adapts all trainable weights to its own geometry. For replacement family $\mathcal B$, the retrained gap is
\begin{equation}
 N_{\mathcal B}=\mathbb E_r
   \left[S(\Alg(\mathrm{Full};r))-S(\Alg(\mathcal B;r))\right],
 \label{eq:necessity}
\end{equation}
where $r$ indexes paired splits and seeds and $\Alg$ denotes fitting and model selection under the shared protocol. Unlike the frozen-checkpoint shock in Eq.~\eqref{eq:reliance}, Eq.~\eqref{eq:necessity} measures the advantage persisting after full adaptation.

\paragraph{Matrix capacity and variation across edges.}
Full predicts a general $d\times d$ restriction matrix for each edge endpoint from concatenated endpoint embeddings. Identity fixes all maps to the identity, and Diagonal predicts edge-conditioned diagonal matrices that scale stalk coordinates without mixing them. Because both use fewer parameters than Full, Capacity-F pairs identity propagation with a residual two-layer node-wise MLP (a node-wise channel adapter) matching Full's parameter count.

Shared-F retains full-matrix freedom while sharing one learned restriction-map pair across all edges in each layer. Shared-C adds a node-wise channel adapter to match Full's parameter count, contrasting equal budgets spent on edge-conditioned maps versus a shared map pair plus node-wise capacity (Table~\ref{tab:budgets}, Appendix~\ref{app:families}).

\paragraph{Map variation and assignment to edges.}
When sharing hurts performance, two effects remain confounded: map diversity (distinct matrices across edges) and native edge assignment (matching each pair to its generating edge). Fixed-F and Resampled-F keep Full's exact map predictor and parameter count while permuting which edge receives each predicted pair. Fixed-F uses one fixed permutation per layer throughout training and evaluation. Resampled-F draws a fresh permutation on every forward pass. Appendix~\ref{app:supp-results} compares single and averaged predictions across assignment draws, and Appendix~\ref{app:newchecks} examines gradient variance during training.

\paragraph{Comparison with the strongest alternative.}
If an alternative matches or beats Full, it can replace the full map family under that protocol. The Necessary criterion tests whether Full beats the entire control suite $\mathcal C$ via the minimum paired mean advantage:
\begin{equation}
 U_{\mathcal C}=\min_{\mathcal B\in\mathcal C}N_{\mathcal B},
 \label{eq:suite}
\end{equation}
where $\mathcal C$ is the specified set of replacement families. A positive minimum $U_{\mathcal C}>0$ satisfies the Necessary criterion in Figure~\ref{fig:claims}, always relative to $\mathcal C$ since adding a stronger baseline can lower $U_{\mathcal C}$. We report paired differences and simultaneous bounds for the minimum across controls, stating near-Full differences on the task scale rather than claiming equivalence without a prespecified tolerance.

\section{A Trained Model Can Rely on Replaceable Maps}
\label{sec:theory}

Consider a node receiving one message $x$ transformed by a linear map $T$. A linear readout with weights $w$ scores this message as $w^\top T x$. Replacing $T$ with the identity at a fixed checkpoint changes the score to $w^\top x$. A linear readout trained from scratch with identity maps can instead set its weights to $T^\top w$ and reproduce every score. Once a node aggregates messages from several neighbors, a single readout must read messages arriving through different maps together.

\subsection{Map Changes That Preserve the Readout Score}

For $G$ neighbor messages of dimension $d$, applying a linear readout to the mean-aggregated transformed messages projects each message onto an effective direction and averages the scalar scores:
\begin{equation}
 f_{T,w}(X)
 =w^\top\!\left(\frac{1}{G}\sum_g T_gX_g\right)
 =\frac{1}{G}\sum_g q_g^\top X_g ,
 \label{eq:effective}
\end{equation}
where $g\in\{0,\ldots,G-1\}$ indexes neighbors, $w\in\mathbb R^d$ holds the linear readout weights, $X_g\in\mathbb R^d$ is neighbor $g$'s message, $T_g\in\mathbb R^{d\times d}$ is its transport map, $q_g=T_g^\top w$ is the effective direction read in that message's coordinates, and $X,T$ collect the messages and maps. Under identity maps, a refitted linear readout reproduces $f_{T,w}(X)$ for all inputs if and only if all $q_g$ coincide. With $w$ fixed, modifying $T_g$ preserves every score whenever $q_g$ is unchanged.

\begin{theorem}[Map directions invisible to a fixed readout]
\label{thm:nullspace}
For fixed nonzero readout weights $w\in\mathbb R^d$ and a specified effective direction $q_g$, every unrestricted map realizing $q_g$ has the form
\begin{equation}
 T_g=\frac{wq_g^\top}{\|w\|_2^2}+B_g ,
 \qquad B_g^\top w=0 ,
 \label{eq:nullmaps}
\end{equation}
where $B_g$ is any matrix whose columns are orthogonal to $w$. This family has dimension $d(d-1)$ per map.
\end{theorem}
\begin{proof}
The first term in Eq.~\eqref{eq:nullmaps} realizes $q_g$. Subtracting it from any other map with the same $q_g$ leaves one vector orthogonal to $w$ in each of its $d$ columns, and each orthogonal vector has $d-1$ free coordinates.
\end{proof}

Even scores across all inputs leave $B_g$ undetermined, so updates inside this subspace do not alter the fixed readout's predictions. Appendix~\ref{app:newchecks} examines how the local nullity changes with readout architecture in a one-map experiment.

\subsection{Class Signal After Map Removal}

Preserving accuracy is weaker than reproducing every score in Eq.~\eqref{eq:effective}. To quantify what refitting recovers, consider neighbors observing the same noisy binary class signal in rotated local frames. Aligning frames before averaging adds the signals coherently along one axis, while identity maps average unaligned directions.

For each neighbor $g$, we model its message as
\begin{equation}
 X_g=R_g(Ys e_1+\varepsilon_g), \qquad
 \varepsilon_g\overset{\mathrm{iid}}{\sim}\mathcal N(0,\sigma^2 I_d),
 \label{eq:frame-data}
\end{equation}
where $Y\in\{-1,+1\}$ is a balanced binary label, $e_1\in\mathbb R^d$ is the first standard basis vector, $s>0$ sets the signal amplitude along $e_1$, and $\varepsilon_g$ has coordinate standard deviation $\sigma>0$ ($I_d$ is the identity matrix). For $g=0,\ldots,G-1$, the orthogonal matrix $R_g$ rotates $e_1$ to angle $2\pi\eta g/G$ in the first two coordinates ($G,d\ge2$). At angular spread $\eta=0$ all neighbor signal directions coincide, and at $\eta=1$ they are evenly spaced around a circle.

Oracle alignment maps $R_g^\top$ invert the frame rotations to align every signal with $e_1$, mirroring connection Laplacians and vector diffusion \citep{singer2012vector,barbero2022connection}. Under identity maps, signal directions average to $m_\eta=G^{-1}\sum_g R_ge_1$, so $e_1^\top m_\eta$ gives the signal left along the original readout and $\|m_\eta\|_2$ gives the signal available to a refitted readout.

\begin{theorem}[Accuracy before and after refitting]
\label{thm:frame}
For the frame model in Eq.~\eqref{eq:frame-data}, the three population accuracies are
\begin{equation}
\begin{aligned}
 A_{\rm aligned} &= \Phi(c),\\
 A_{\rm fixed}   &= \Phi(c\,e_1^\top m_\eta),\\
 A_{\rm refit}   &= \Phi(c\,\|m_\eta\|_2),
\end{aligned}
\label{eq:frame-accuracy}
\end{equation}
where $c=s\sqrt G/\sigma$ is the aligned signal-to-noise ratio and $\Phi$ is the standard normal cumulative distribution function. Here $A_{\rm aligned}$ uses oracle alignment with its Bayes-optimal linear readout, while $A_{\rm fixed}$ and $A_{\rm refit}$ use identity maps with that readout held fixed and refitted to optimality, respectively.
\end{theorem}
\begin{proof}
Conditioned on $Y$, the aligned average has mean $Ys e_1$ and covariance $\sigma^2 I_d/G$. Because each $R_g$ is orthogonal, the unaligned average retains the same isotropic covariance with mean $Ys m_\eta$. The signal along the fixed readout is $s e_1^\top m_\eta$, and the optimal refitted readout extracts $s\|m_\eta\|_2$. Dividing each projected signal by the noise standard deviation $\sigma/\sqrt{G}$ yields Eq.~\eqref{eq:frame-accuracy}, giving chance accuracy under identity maps whenever $m_\eta=0$.
\end{proof}

\begin{corollary}\label{cor:gap}
The accuracy loss from replacing oracle maps with identity at a fixed checkpoint is at least as large as the loss after refitting the readout, with strict inequality whenever $\|m_\eta\|_2>e_1^\top m_\eta$.
\end{corollary}
Because $e_1^\top m_\eta \le \|m_\eta\|_2$, increasing $\eta$ across intermediate values both shortens the mean and rotates it away from $e_1$, while refitting realigns the readout with the surviving mean (Figure~\ref{fig:frameplot}). At $\eta=1$, the rotated signals cancel completely, leaving identical class-conditional distributions and driving even the optimal identity-map readout to chance.

\begin{figure}[t]
\centering
\includegraphics[width=\columnwidth]{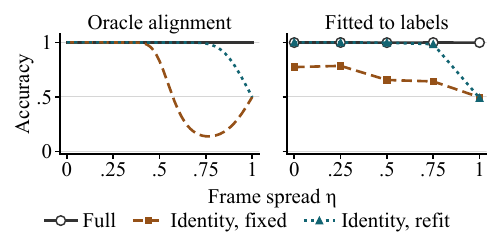}
\par\noindent\parbox{\columnwidth}{\footnotesize\textit{Note.} Left curves give population accuracies. Right points give means over 20 seeds (SDs in Table~\ref{tab:core-frame}).}
\caption{Accuracy as neighbor frames rotate apart. Oracle alignment uses the Bayes readout (left, Theorem~\ref{thm:frame}). Full learns one map per frame and a shared readout from labels (right). Identity keeps the readout fixed or refits it to unaligned messages.}
\label{fig:frameplot}
\end{figure}

\begin{table}[t]
\centering
\small
\setlength{\tabcolsep}{2pt}
\begin{tabular*}{\columnwidth}{@{\extracolsep{\fill}}rrrrrr@{}}
\toprule
& & \multicolumn{3}{c}{Maps at the Full checkpoint} & Retrained\\
\cmidrule(lr){3-5}
$\eta$ & Full & Identity & Shuffle & Mean & Identity\\
\midrule
$0$   & \resultstat{.998}{.003} & \resultstat{.773}{.304} & \resultstat{.999}{.003} & \resultstat{1.000}{.001} & \resultstat{.999}{.002}\\
$.25$ & \resultstat{.999}{.003} & \resultstat{.784}{.253} & \resultstat{.999}{.002} & \resultstat{.999}{.002} & \resultstat{.999}{.002}\\
$.50$ & \resultstat{.997}{.004} & \resultstat{.655}{.232} & \resultstat{.865}{.182} & \resultstat{.943}{.075} & \resultstat{.997}{.004}\\
$.75$ & \resultstat{.998}{.002} & \resultstat{.642}{.174} & \resultstat{.662}{.268} & \resultstat{.747}{.143} & \resultstat{.983}{.008}\\
$1$   & \resultstat{.998}{.003} & \resultstat{.497}{.030} & \resultstat{.499}{.247} & \resultstat{.493}{.036} & \resultstat{.491}{.031}\\
\bottomrule
\end{tabular*}
\par\noindent\parbox{\columnwidth}{\footnotesize\textit{Note.} Subscripts give the sample standard deviation of the scores across 20 seeds.}
\caption{Test accuracy on the label-only frame task. At each angular spread $\eta$, the fitted Full readout is held fixed while its learned maps are replaced by identities, shuffled maps, or their mean. Retrained Identity fits a linear readout from scratch on unaligned averages.}
\label{tab:core-frame}
\end{table}

\paragraph{Checkpoint reliance after training on class labels.}
To test label-only training without oracle supervision, we swept $\eta\in[0,1]$ with $G=8$, $d=6$, and fixed signal and noise levels. Full learned one map per frame jointly with a linear readout from class labels, while Identity fitted a linear readout on unaligned averages over the same splits. At each spread, we intervened on the frozen Full checkpoint by replacing its learned maps with identities, shuffled maps, or their mean. At $\eta=0$, where the oracle alignment leaves the signal unchanged, identity replacement dropped Full's mean accuracy from $.998$ to $.773$, while shuffled and mean maps stayed near $1.000$ and retrained Identity reached $.999$ (Table~\ref{tab:core-frame}). At $\eta=1$, unaligned signals cancel, retrained Identity fell near chance, and Full kept $.998$ accuracy (Tables~\ref{tab:calibration} and~\ref{tab:label-only-full}, Appendix~\ref{app:frames}).

\section{Experiments on Node Classification}
\label{sec:graph-experiments}

In the synthetic frame task, identity replacement hurt a fixed checkpoint even when a retrained identity readout solved the task. On real graphs, class labels supervise predictions without specifying ground-truth transport maps, so we apply LRN-Evaluation to compare fixed-checkpoint interventions against matched retraining across all eight map designs.

\begin{table*}[t]
\centering
\setlength{\tabcolsep}{3.0pt}
\begin{tabular}{@{}llccccc@{}}
\toprule
& & \multicolumn{2}{c}{Accuracy} & \multicolumn{3}{c}{AUROC}\\
\cmidrule(lr){3-4}\cmidrule(l){5-7}
Model & Maps used during training & Roman-Empire & Amazon-Ratings & Questions & Tolokers & Minesweeper\\
\midrule
Full & Native edge-conditioned matrices & \resultstat{\mathbf{.845}}{.005} & \resultstat{\mathbf{.478}}{.009} & \resultstat{.772}{.010} & \resultstat{.846}{.006} & \resultstat{.935}{.010}\\
\addlinespace[2pt]
Identity & Fixed identity & \resultstat{.774}{.006} & \resultstat{.431}{.006} & \resultstat{.730}{.015} & \resultstat{.841}{.009} & \resultstat{.945}{.007}\\
Diagonal & Edge-conditioned diagonal matrices & \resultstat{.816}{.008} & \resultstat{.461}{.005} & \resultstat{.752}{.014} & \resultstat{.840}{.006} & \resultstat{.938}{.005}\\
Capacity-F & Identity and node-wise adapter & \resultstat{.795}{.008} & \resultstat{.422}{.009} & \resultstat{.739}{.012} & \resultstat{.839}{.008} & \resultstat{.945}{.006}\\
\addlinespace[2pt]
Shared-F & One matrix pair per layer & \resultstat{.779}{.007} & \resultstat{.446}{.006} & \resultstat{.729}{.012} & \resultstat{.840}{.005} & \resultstat{\mathbf{.948}}{.007}\\
Shared-C & Shared pair and node-wise adapter & \resultstat{.805}{.006} & \resultstat{.435}{.008} & \resultstat{.741}{.012} & \resultstat{.842}{.009} & \resultstat{.946}{.006}\\
\addlinespace[2pt]
Fixed-F & Stable reassignment of map pairs & \resultstat{.719}{.014} & \resultstat{.454}{.007} & \resultstat{\mathbf{.776}}{.010} & \resultstat{.844}{.009} & \resultstat{.905}{.007}\\
Resampled-F & Fresh assignment each forward pass & \resultstat{.777}{.006} & \resultstat{.476}{.008} & \resultstat{.775}{.009} & \resultstat{\mathbf{.849}}{.008} & \resultstat{.926}{.013}\\
\bottomrule
\multicolumn{7}{@{}l}{\footnotesize\textit{Note.} Scores are means. Subscripts give sample standard deviations across the ten official splits.}\\
\end{tabular}
\caption{Matched retraining performance across eight DNSD map designs trained from scratch. Full achieves the highest mean accuracy on Roman-Empire and Amazon-Ratings, while Shared-F leads on Minesweeper. Bold marks the highest mean per column. Table~\ref{tab:paired} (Appendix~\ref{app:supp-results}) reports paired split-wise differences.}
\label{tab:main}
\end{table*}

\subsection{Experimental setup}
\paragraph{Datasets.}
Our five primary benchmarks are Roman-Empire, Amazon-Ratings, Questions, Tolokers, and Minesweeper \citep{platonov2023critical}, loaded via PyTorch Geometric \citep{fey2019pyg}. Following the benchmark protocol, Roman-Empire and Amazon-Ratings use accuracy while Questions, Tolokers, and Minesweeper use AUROC. Amazon-Ratings derives from Amazon co-purchases \citep{leskovec2007viral} and Tolokers from Toloka worker interactions \citep{likhobaba2023toloker}. We evaluate DNSD on all ten official splits per benchmark (Table~\ref{tab:datasets}, Appendix~\ref{app:protocol}) and use WebKB Texas for our directed sheaf transfer check \citep{craven1998webkb,pei2020geom}.

\paragraph{Protocol and statistical summaries.}
All eight DNSD designs (Table~\ref{tab:main}) share the same backbone (eight layers, 64 hidden channels, stalk dimension $d=8$) and train with Adam for up to 500 epochs (weight decay $5\times10^{-4}$, ReduceLROnPlateau, gradient clipping at 1.0, early-stopping patience 100). Each family selects its learning rate from $\{.001,.003,.01,.03\}$ and its checkpoint strictly by validation score, totaling 1,600 DNSD grid fits plus 80 NSD fits across all splits (Appendices~\ref{app:selection} and~\ref{app:protocol}). Resampled-F averages logits over eight independent assignments at validation and test time. Checkpoint interventions evaluate Full checkpoints trained at the NSD reference learning rate of $.01$ \citep{bodnar2022neural} on Roman-Empire and Minesweeper \citep{bourgerie2026deep} and at the validation-selected rate on the other three tasks. Paired differences across the ten splits use two-sided 95\% Student-$t$ intervals (Appendix~\ref{app:procedures}), with simultaneous bounds across the seven controls in Appendix~\ref{app:supp-results}.

\subsection{The benefit of edge-conditioned maps varies by task}

\begin{table}[t]
\centering
\small
\setlength{\tabcolsep}{1.2pt}
\begin{tabular*}{\columnwidth}{@{\extracolsep{\fill}}lrrrrr@{}}
\toprule
& & \multicolumn{3}{c}{Checkpoint score drop} & Retrained\\
\cmidrule(lr){3-5}
Dataset & $M$ & Identity & Shuffle & Mean & Full $-$ Id.\\
\midrule
\begin{tabular}[c]{@{}l@{}}Roman-\\Empire\end{tabular} & \resultstat{.046}{.003} & \resultstat{.3231}{.0236} & \resultstat{.2599}{.0121} & \resultstat{.2490}{.0305} & \resultstat{+.070}{.007}\\
\begin{tabular}[c]{@{}l@{}}Amazon-\\Ratings\end{tabular} & \resultstat{.026}{.008} & \resultstat{.2371}{.0130} & \resultstat{.1255}{.0456} & \resultstat{.1328}{.0635} & \resultstat{+.047}{.010}\\
Questions & \resultstat{.019}{.005} & \resultstat{.1223}{.0221} & \resultstat{.0254}{.0125} & \resultstat{.0192}{.0098} & \resultstat{+.042}{.012}\\
Tolokers & \resultstat{.029}{.013} & \resultstat{.1940}{.0282} & \resultstat{.0943}{.0444} & \resultstat{.0905}{.0437} & \resultstat{+.006}{.008}\\
Minesweeper & \resultstat{.045}{.004} & \resultstat{.3714}{.1052} & \resultstat{.1101}{.0518} & \resultstat{.0715}{.0322} & \resultstat{-.010}{.009}\\
\bottomrule
\end{tabular*}
\par\noindent\parbox{\columnwidth}{\footnotesize\textit{Note.} Subscripts give sample standard deviations across the ten official splits.}
\caption{LRN-Evaluation summary across five benchmarks: map-predictor movement $M$ (RMS parameter change), checkpoint score drops (native minus intervened), and paired Full-minus-Identity differences after matched retraining. Roman-Empire and Minesweeper use Full at learning rate $.01$ for interventions and validation-selected rates for retraining. Scores report accuracy on the first two datasets and AUROC on the remaining three.}
\label{tab:core-audit}
\end{table}

Table~\ref{tab:core-audit} summarizes all three stages. Optimization moves map-predictor weights ($M>0$) in every Full run, and both identity replacement and cross-edge shuffling lower test scores on all 50 fitted checkpoints, even though shuffling keeps each layer's exact multiset of predicted pairs. Individual drops reach at least $.0975$ under identity replacement and $.0042$ under shuffling, while layer-mean replacement flips $37.3\%$ of predictions on Roman-Empire and $54.2\%$ on Amazon-Ratings (Table~\ref{tab:interventions}, Appendix~\ref{app:supp-results}). The retrained Full-minus-Identity differences are smaller on all five benchmarks and change sign on Minesweeper. Figure~\ref{fig:pairedplot} compares Full against two retrained alternatives retaining full-matrix maps: layer-wide sharing (Shared-F) and fresh per-pass reassignment (Resampled-F).

\begin{figure}[t]
\centering
\includegraphics[width=\columnwidth]{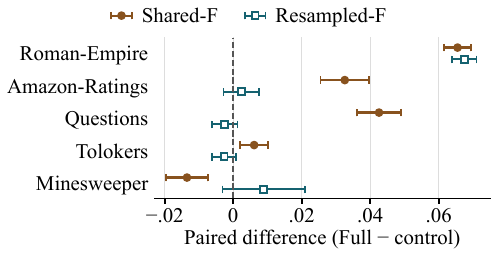}
\par\noindent\parbox{\columnwidth}{\footnotesize\textit{Note.} Points give mean differences over ten paired splits. Horizontal bars show paired 95\% CIs.}
\caption{Full minus Shared-F or Resampled-F after matched retraining. Positive differences favor Full. The first two rows show accuracy and the last three show AUROC.}
\label{fig:pairedplot}
\end{figure}

\paragraph{Full outperforms all seven alternatives on Roman-Empire.}
On Roman-Empire, Diagonal is the strongest of the seven alternatives, keeping native edge assignment while scaling coordinates without mixing them. Full beats Diagonal by 2.87 percentage points, with a positive simultaneous lower bound for its advantage over the strongest control (Appendix~\ref{app:supp-results}). Moving parameters into node-wise adapters (Capacity-F and Shared-C) lowers accuracy further. Fixed-F and Resampled-F preserve Full's matrix capacity and exact parameter count, yet both trail Full substantially even when Fixed-F adapts to one stable permutation throughout training. Appendix~\ref{app:newchecks} examines Full's advantage over reassigned maps across training budgets, gradient-variance controls and DNSD propagation branches.

\begin{figure}[t]
\centering
\includegraphics[width=\columnwidth]{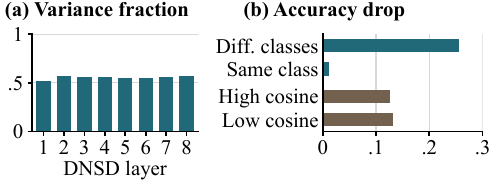}
\par\noindent\parbox{\columnwidth}{\footnotesize\textit{Note.} Bars give means over ten splits. Appendix~\ref{app:mechanism} gives SDs by layer (Figure~\ref{fig:layers}) and edge subset (Table~\ref{tab:strata}).}
\caption{Roman-Empire edge transport. (a) Fraction of variance $H$ explained by ordered endpoint class pairs at each layer. (b) Accuracy loss after replacing maps on each edge subset by their layer means. The same- and different-label subsets differ in size. The cosine split is balanced.}
\label{fig:role-edges}
\end{figure}

\paragraph{Endpoint roles explain about half the variation in each layer.}
On Roman-Empire, neighbors of the same word often hold different syntactic roles. When we group directed edges in each fitted layer by their ordered endpoint class pair, the variance fraction explained by class-pair means is
\begin{equation}
 H=\frac{\sum_c n_c\|\bar A_c-\bar A\|_F^2}
          {\sum_e\|A_e-\bar A\|_F^2},
 \label{eq:classpair}
\end{equation}
where $A_e$ is the edge transport map from Eq.~\eqref{eq:action}, $\bar A_c$ is its mean over the $n_c$ edges in class pair $c$, and $\bar A$ is the layer mean. Class-pair means explain about half the transport variance in every layer (Figure~\ref{fig:role-edges}a), while the rest reflects variation within each role pair.

Replacing source and target restriction maps with layer means on one edge subset at a time causes a much larger accuracy drop on heterophilous (different-role) edges than on homophilous (same-role) edges (Figure~\ref{fig:role-edges}b). Because different-role edges make up $95.3\%$ of the graph, this drop reflects both group size and per-edge influence. Splitting edges at the median cosine similarity of endpoint features yields two nearly equal strata with similar losses ($.1325$ versus $.1253$, Appendix~\ref{app:mechanism}, Table~\ref{tab:strata}).

\paragraph{Shared maps improve AUROC on Minesweeper.}
On Minesweeper, identity replacement at a frozen Full checkpoint drops AUROC far more than shuffling pairs across the grid (Table~\ref{tab:core-audit}). After matched retraining, Identity surpasses Full on average and Shared-F achieves the highest AUROC overall. The simultaneous 95\% interval for Shared-F minus Full is positive (Appendix~\ref{app:supp-results}). One shared restriction-map pair per layer thus improves prediction using 1,024 map parameters versus 132,096 in Full (Table~\ref{tab:budgets}, Appendix~\ref{app:families}).

\paragraph{Resampling stays within .003 of Full on three tasks.}
On Amazon-Ratings and Questions, sharing one map pair per layer lowers performance even when Shared-C restores Full's parameter count. Resampled-F recovers most of that gap on both tasks (Figure~\ref{fig:pairedplot} and Table~\ref{tab:paired} in Appendix~\ref{app:supp-results}), staying within $.003$ of Full on average with edge assignments resampled on every forward pass (paired means $+.0024$ and $-.0025$).

On Tolokers, Resampled-F again stays within $.003$ of Full (paired mean $-.0027$), and retrained Identity is competitive with a paired 95\% confidence interval for Full minus Identity spanning zero (Appendix~\ref{app:supp-results}, Table~\ref{tab:paired}). Claiming formal equivalence on these three tasks would require a prespecified tolerance.

\begin{table}[t]
\centering\small
\begin{tabular}{@{}lcc@{}}
\toprule
NSD maps & Roman-Empire & Minesweeper\\
         & Accuracy & AUROC\\
\midrule
Identity & \resultstat{.7618}{.0077} & \resultstat{.9239}{.0041}\\
Diagonal & \resultstat{.7751}{.0062} & \resultstat{.9239}{.0073}\\
Full & \resultstat{.6951}{.1433} & \resultstat{.9008}{.0088}\\
Capacity-F & \resultstat{.8040}{.0072} & \resultstat{.9358}{.0072}\\
\bottomrule
\end{tabular}
\par\noindent\parbox{\columnwidth}{\footnotesize\textit{Note.} Subscripts give sample standard deviations across the ten official splits.}
\caption{Neural Sheaf Diffusion (NSD) at its reference learning rate of $.01$ on Roman-Empire and Minesweeper. Capacity-F pairs identity maps with a node-wise adapter.}
\label{tab:nsd}
\end{table}

\paragraph{The preferred map design also depends on the model and training protocol.}
Although native edge-conditioned DNSD beats all seven controls on Roman-Empire under family-specific rate selection, standard NSD at its reference rate shows the reverse: Capacity-F beats Full NSD on Roman-Empire and Minesweeper (Table~\ref{tab:nsd}). Full NSD's high Roman-Empire variance reflects instability at this rate, including one collapsed split, yet Capacity-F beats Full NSD on all ten splits of both tasks (medians $.8032$ vs.\ $.7576$ and $.9353$ vs.\ $.9005$). The two protocols differ in both operator and learning-rate selection.

To probe a third operator, we tested edge assignment in the complex-valued Directed Sheaf Neural Network (DSNN) \citep{fiorini2026reloaded} on WebKB Texas \citep{craven1998webkb,pei2020geom}. Under its published setup, fixed and native assignment have the same mean accuracy and per-pass resampling has a slightly higher mean, with a paired 95\% interval spanning zero (Table~\ref{tab:dsnn}, Appendix~\ref{app:supp-results}).

\section{Discussion}

\paragraph{Interpreting learned geometry.}
Interpreting a learned sheaf starts from the induced edge transport map $A_{u\to v}$ rather than raw restriction matrices, since edge-stalk rotations alter both matrices without changing $A_{u\to v}$. Claiming that $A_{u\to v}$ recovers a specific node relation requires an external geometric check, such as aligning held-out vectors across known frames. With node labels alone, matched retraining compares the predictive value of alternative map designs under the specified protocol.

\paragraph{Choosing how much to share.}
Where one layer-shared pair loses accuracy while freshly reassigned pairs approach Full, intermediate sharing warrants study. A layer can learn a small bank of map pairs and route each edge to one pair. Sweeping bank size against validation scores and total parameters would test how much diversity matches Full.

\paragraph{Predicting which tasks benefit.}
LRN-Evaluation measures edge-conditioned map value after training the control suite, while forecasting that gain beforehand remains open. Candidate predictors include the empirical class-pair mixing matrix from training labels and adjacent-node feature similarity. Testing these predictors on held-out graphs across operators would clarify which designs transfer.

\section{Conclusion}

A trained SNN can rely heavily on learned restriction maps even when an identity-map or layer-shared model performs better after retraining. In the rotated-frame Gaussian model, a refitted readout recovers class signal that a fixed readout misses after map removal. LRN-Evaluation separates map-weight movement, checkpoint reliance, and post-retraining value on real graphs. Across five benchmarks, shuffling maps degraded every Full DNSD checkpoint, yet matched retraining separated three regimes: layer-shared maps improved Minesweeper AUROC, per-pass resampling stayed close to Full on Amazon-Ratings, Questions, and Tolokers, and native edge-conditioned maps beat all seven controls on Roman-Empire. These rankings concern the tested tasks, propagation rules, training protocols, and control families.

\renewcommand{\bibfont}{\small}
\bibliographystyle{plainnat}
\bibliography{references}

@article{hansen2019spectral,
  author  = {Hansen, Jakob and Ghrist, Robert},
  title   = {Toward a Spectral Theory of Cellular Sheaves},
  journal = {Journal of Applied and Computational Topology},
  year    = {2019},
  volume  = {3},
  pages   = {315--358},
  doi     = {10.1007/s41468-019-00038-7},
  number  = {4},
  url     = {https://doi.org/10.1007/s41468-019-00038-7}
}

@inproceedings{hansen2020sheaf,
  title     = {Sheaf Neural Networks},
  author    = {Hansen, Jakob and Gebhart, Thomas},
  booktitle = {NeurIPS 2020 Workshop on Topological Data Analysis and Beyond},
  year      = {2020},
  note      = {Spotlight presentation},
  url       = {https://neurips.cc/virtual/2020/20055}
}

@inproceedings{bodnar2022neural,
 author = {Bodnar, Cristian and Di Giovanni, Francesco and Chamberlain, Benjamin P. and Li{\`o}, Pietro and Bronstein, Michael},
 booktitle = {Advances in Neural Information Processing Systems},
 doi = {10.52202/068431-1346},
 editor = {S. Koyejo and S. Mohamed and A. Agarwal and D. Belgrave and K. Cho and A. Oh},
 pages = {18527--18541},
 publisher = {Curran Associates, Inc.},
 title = {Neural Sheaf Diffusion: A Topological Perspective on Heterophily and Oversmoothing in {GNNs}},
 url = {https://proceedings.neurips.cc/paper_files/paper/2022/file/75c45fca2aa416ada062b26cc4fb7641-Paper-Conference.pdf},
 volume = {35},
 year = {2022}
}

@inproceedings{barbero2022connection,
  title = 	 {Sheaf Neural Networks with Connection Laplacians},
  author =       {Barbero, Federico and Bodnar, Cristian and S\'aez de Oc\'ariz Borde, Haitz and Bronstein, Michael and Veli\v{c}kovi\'c, Petar and Li\`o, Pietro},
  booktitle = 	 {Proceedings of Topological, Algebraic, and Geometric Learning Workshops 2022},
  pages = 	 {28--36},
  year = 	 {2022},
  editor = 	 {Cloninger, Alexander and Doster, Timothy and Emerson, Tegan and Kaul, Manohar and Ktena, Ira and Kvinge, Henry and Miolane, Nina and Rieck, Bastian and Tymochko, Sarah and Wolf, Guy},
  volume = 	 {196},
  series = 	 {Proceedings of Machine Learning Research},
  month = 	 {25 Feb--22 Jul},
  publisher =    {PMLR},
  url = 	 {https://proceedings.mlr.press/v196/barbero22a.html}
}

@inproceedings{bamberger2025bundle,
 author = {Bamberger, Jacob and Barbero, Federico and Dong, Xiaowen and Bronstein, Michael},
 booktitle = {International Conference on Learning Representations},
 title = {Bundle Neural Network for message diffusion on graphs},
 url = {https://iclr.cc/virtual/2025/poster/28114},
 year = {2025}
}

@misc{borgi2026polynsd,
      title={Polynomial Neural Sheaf Diffusion: A Spectral Filtering Approach on Cellular Sheaves}, 
      author={Alessio Borgi and Fabrizio Silvestri and Pietro Li{\`o}},
      year={2026},
      eprint={2512.00242},
      archivePrefix={arXiv},
      primaryClass={cs.LG},
      url={https://arxiv.org/abs/2512.00242}, 
}

@misc{bourgerie2026deep,
      title={Deep Neural Sheaf Diffusion}, 
        author= {Bourgerie, R{\'e}mi and
              Girdzijauskas, {\v{S}}ar{\=u}nas and
              Fodor, Viktoria},
      year={2026},
      eprint={2605.19021},
      archivePrefix={arXiv},
      primaryClass={cs.LG},
      url={https://arxiv.org/abs/2605.19021}, 
}

@inproceedings{choi2026pacbayes,
  title     = {Sheaf Graph Neural Networks via {PAC-Bayes} Spectral Optimization},
  author    = {Choi, Y. and Choi, J. and Ko, T. and Kim, J. and Kim, C.-K.},
  booktitle = {Proceedings of the AAAI Conference on Artificial Intelligence},
  pages     = {20570--20578},
  year      = {2026},
  volume    = {40},
  number    = {25},
  doi       = {10.1609/aaai.v40i25.39193},
  url       = {https://doi.org/10.1609/aaai.v40i25.39193}
}

@inproceedings{fiorini2026reloaded,
  title     = {Sheaves Reloaded: A Direction Awakening},
  author    = {Fiorini, Stefano and
               Aktas, Hakan Emre and
               Duta, Iulia and
               Morerio, Pietro and
               Del Bue, Alessio and
               Lio, Pietro and
               Coniglio, Stefano},
  booktitle = {International Conference on Learning Representations},
  year      = {2026},
  url       = {https://iclr.cc/virtual/2026/poster/10007988}
}

@inproceedings{ribeiro2026cooperative,
  title     = {Cooperative Sheaf Neural Networks},
  author    = {Ribeiro, Andr{\'e} and
               Ten{\'o}rio, Ana Luiza and
               Belieni, Juan and
               Souza, Amauri H. and
               Mesquita, Diego},
  booktitle = {International Conference on Learning Representations},
  year      = {2026},
  url       = {https://iclr.cc/virtual/2026/poster/10011050}
}

@article{battiloro2024tangent,
  author={Battiloro, Claudio and Wang, Zhiyang and Riess, Hans and Di Lorenzo, Paolo and Ribeiro, Alejandro},
  journal={IEEE Transactions on Signal Processing}, 
  title={Tangent Bundle Convolutional Learning: From Manifolds to Cellular Sheaves and Back}, 
  year={2024},
  volume={72},
  number={},
  pages={1892--1909},
  doi={10.1109/TSP.2024.3379862}
  }

@inproceedings{caralt2024joint,
  title = 	 {Joint Diffusion Processes as an Inductive Bias in Sheaf Neural Networks},
  author =       {Hernandez Caralt, Ferran and Bern{\'a}rdez Gil, Guillermo and Duta, Iulia and Li{\`o}, Pietro and Alarc{\'o}n Cot, Eduard},
  booktitle = 	 {Proceedings of the Geometry-grounded Representation Learning and Generative Modeling Workshop (GRaM)},
  pages = 	 {249--263},
  year = 	 {2024},
  editor = 	 {Vadgama, Sharvaree and Bekkers, Erik and Pouplin, Alison and Kaba, Sekou-Oumar and Walters, Robin and Lawrence, Hannah and Emerson, Tegan and Kvinge, Henry and Tomczak, Jakub and Jegelka, Stephanie},
  volume = 	 {251},
  series = 	 {Proceedings of Machine Learning Research},
  month = 	 {29 Jul},
  publisher =    {PMLR},
  url = 	 {https://proceedings.mlr.press/v251/hernandez-caralt24a.html}
}

@inproceedings{hernandez2026necessity,
  author = {Hernandez Caralt, Ferran and
            Gonz{\`a}lez i Catal{\`a}, Mar and
            Bazaga, Adri{\'a}n and
            Li{\`o}, Pietro},
  title = {On the Necessity of Learnable Sheaf Laplacians},
  booktitle = {ICLR 2026 Workshop on Geometry-grounded Representation Learning and Generative Modeling ({GRaM}), Tiny Paper Track},
  url = {https://openreview.net/forum?id=IuXTYO5ggp},
  year = {2026}
}

@misc{donmez2026degeneracy,
      title={Oversmoothing as Representation Degeneracy in Neural Sheaf Diffusion}, 
      author={Arif D{\"o}nmez and Axel Mosig and Ellen Fritsche and Katharina Koch},
      year={2026},
      eprint={2605.11178},
      archivePrefix={arXiv},
      primaryClass={cs.LG},
      url={https://arxiv.org/abs/2605.11178}, 
}

@inproceedings{hooker2019benchmark,
title = {A Benchmark for Interpretability Methods in Deep Neural Networks},
author = {Hooker, Sara and Erhan, Dumitru and Kindermans, Pieter-Jan and Kim, Been},
booktitle = {Advances in Neural Information Processing Systems},
editor = {H. Wallach and H. Larochelle and A. Beygelzimer and F. d\textquotesingle Alch\'{e}-Buc and E. Fox and R. Garnett},
pages = {9737--9748},
year = {2019},
volume={32},
publisher = {Curran Associates, Inc.},
url = {http://papers.nips.cc/paper/9167-a-benchmark-for-interpretability-methods-in-deep-neural-networks.pdf}
}

@article{fisher2019all,
  author  = {Aaron Fisher and Cynthia Rudin and Francesca Dominici},
  title   = {All Models are Wrong, but Many are Useful: Learning a Variable's Importance by Studying an Entire Class of Prediction Models Simultaneously},
  journal = {Journal of Machine Learning Research},
  year    = {2019},
  volume  = {20},
  number  = {177},
  pages   = {1--81},
  url     = {http://jmlr.org/papers/v20/18-760.html}
}

@inproceedings{platonov2023critical,
  title     = {A Critical Look at the Evaluation of {GNNs} under Heterophily:
               Are We Really Making Progress?},
  author    = {Platonov, Oleg and
               Kuznedelev, Denis and
               Diskin, Michael and
               Babenko, Artem and
               Prokhorenkova, Liudmila},
  booktitle = {International Conference on Learning Representations},
  year      = {2023},
  url       = {https://openreview.net/forum?id=tJbbQfw-5wv}
}

@inproceedings{fey2019pyg,
      title={Fast Graph Representation Learning with PyTorch Geometric}, 
      booktitle = {ICLR Workshop on Representation Learning on Graphs and Manifolds},
      author={Matthias Fey and Jan Eric Lenssen},
      year={2019},
      url = {https://rlgm.github.io/papers/2.pdf},
}

@article{holm1979simple,
  author  = {Holm, Sture},
  title   = {A Simple Sequentially Rejective Multiple Test Procedure},
  journal = {Scandinavian Journal of Statistics},
  year    = {1979},
  volume  = {6},
  number  = {2},
  pages   = {65--70},
  url     = {http://www.jstor.org/stable/4615733}
}

@article{singer2012vector,
author = {Singer, Amit and Wu, Hau-Tieng},
year = {2012},
number = {8},
pages = {1067--1144},
title = {Vector diffusion maps and the connection Laplacian},
volume = {65},
journal = {Communications on Pure and Applied Mathematics},
doi = {10.1002/cpa.21395}
}

@misc{gillespie2024bayesian,
      title={Bayesian Sheaf Neural Networks}, 
      author={Patrick Gillespie and {Bou Hamdan}, Layal and Ioannis Schizas and
      David L. Boothe and Vasileios Maroulas},
      year={2025},
      eprint={2410.09590},
      archivePrefix={arXiv},
      primaryClass={cs.LG},
      url={https://arxiv.org/abs/2410.09590}, 
}

@InProceedings{zaghen2024nonlinear,
  title = 	 {Sheaf Diffusion Goes Nonlinear: Enhancing {GNNs} with Adaptive Sheaf Laplacians},
  author =       {Zaghen, Olga and Longa, Antonio and Azzolin, Steve and Telyatnikov, Lev and Passerini, Andrea and Li\`o, Pietro},
  booktitle = 	 {Proceedings of the Geometry-grounded Representation Learning and Generative Modeling Workshop (GRaM)},
  pages = 	 {264--276},
  year = 	 {2024},
  editor = 	 {Vadgama, Sharvaree and Bekkers, Erik and Pouplin, Alison and Kaba, Sekou-Oumar and Walters, Robin and Lawrence, Hannah and Emerson, Tegan and Kvinge, Henry and Tomczak, Jakub and Jegelka, Stephanie},
  volume = 	 {251},
  series = 	 {Proceedings of Machine Learning Research},
  month = 	 {29 Jul},
  publisher =    {PMLR},
  url = 	 {https://proceedings.mlr.press/v251/zaghen24a.html},
}

@INPROCEEDINGS{dinino2025restriction,
  author={Di Nino, Leonardo and Barbarossa, Sergio and Di Lorenzo, Paolo},
  booktitle={2024 58th Asilomar Conference on Signals, Systems, and Computers}, 
  title={Learning Sheaf Laplacian Optimizing Restriction Maps}, 
  year={2024},
  volume={},
  number={},
  pages={59--63},
  doi={10.1109/IEEECONF60004.2024.10942997}
  }

@misc{grover2026holonomy,
 author={Grover, Ankit and Bourgerie, R{\'e}mi},
 title={Do Sheaf Neural Networks Use Holonomy? A Measure--Intervene--Control Study},
 year={2026}, howpublished={arXiv:2607.19514}}

@misc{guadagnuolo2026bochner,
 author={Guadagnuolo, Francesco Maria and Borgi, Alessio and Thanou, Dorina and Li{\`o}, Pietro},
 title={A {Bochner Weitzenb{\"o}ck} Decomposition for Cellular Sheaf Laplacians, Curvature and Neural Sheaf Diffusion},
 year={2026}, howpublished={arXiv:2610.28455}}

@misc{fiorini2026inductive,
 author={Fiorini, Stefano and Coppola, Edoardo and Li{\`o}, Pietro},
 title={Benchmarking Sheaf Neural Networks for Inductive Tasks},
 year={2026}, howpublished={arXiv:2608.02558}}

@inproceedings{pei2020geom,
 author={Pei, Hongbin and Wei, Bingzhe and Chang, Kevin Chen-Chuan and Lei, Yu and Yang, Bo},
 title={{Geom-GCN}: Geometric Graph Convolutional Networks},
 booktitle={International Conference on Learning Representations},
 year={2020},
 url={https://openreview.net/forum?id=S1e2agrFvS}}

@inproceedings{craven1998webkb,
  author    = {Craven, Mark and DiPasquo, Dan and Freitag, Dayne and McCallum, Andrew and Mitchell, Tom and Nigam, Kamal and Slattery, Se{\'a}n},
  title     = {Learning to Extract Symbolic Knowledge from the {World Wide Web}},
  booktitle = {Proceedings of the Fifteenth National Conference on Artificial Intelligence},
  year      = {1998},
  pages     = {509--516},
  url       = {https://cdn.aaai.org/AAAI/1998/AAAI98-072.pdf}
}

@article{rozemberczki2021multiscale,
  author  = {Rozemberczki, Benedek and Allen, Carl and Sarkar, Rik},
  title   = {Multi-Scale Attributed Node Embedding},
  journal = {Journal of Complex Networks},
  year    = {2021},
  volume  = {9},
  number  = {2},
  pages   = {cnab014},
  doi     = {10.1093/comnet/cnab014},
  url     = {https://doi.org/10.1093/comnet/cnab014}
}

@article{leskovec2007viral,
  author  = {Leskovec, Jure and Adamic, Lada A. and Huberman, Bernardo A.},
  title   = {The Dynamics of Viral Marketing},
  journal = {ACM Transactions on the Web},
  year    = {2007},
  volume  = {1},
  number  = {1},
  pages   = {5},
  doi     = {10.1145/1232722.1232727},
  url     = {https://doi.org/10.1145/1232722.1232727}
}

@misc{likhobaba2023toloker,
  author       = {Likhobaba, Daniil and Pavlichenko, Nikita and Ustalov, Dmitry},
  title        = {Toloker Graph: Interaction of Crowd Annotators},
  year         = {2023},
  howpublished = {Zenodo},
  doi          = {10.5281/zenodo.7620795},
  url          = {https://doi.org/10.5281/zenodo.7620795}
}

@inproceedings{peng2026spd,
 author={Peng, Yuhan and Dong, Junwen and Zeng, Yuzhi and Li, Hao and Ju, Ce and Feng, Huitao and Taha, Diaaeldin and Wienhard, Anna and Xia, Kelin},
 title={Sheaf Neural Networks on {SPD} Manifolds: Second-Order Geometric Representation Learning},
 booktitle={Proceedings of the 43rd International Conference on Machine Learning},
 year={2026}, volume={306}, pages={97981--98013}, series={Proceedings of Machine Learning Research}, publisher={PMLR}}

@inproceedings{hansen2019smooth,
 author={Hansen, Jakob and Ghrist, Robert},
 title={Learning Sheaf Laplacians from Smooth Signals},
 booktitle={IEEE International Conference on Acoustics, Speech and Signal Processing},
 year={2019}, pages={5446--5450}}

@inproceedings{wiegreffe2019attention,
 author={Wiegreffe, Sarah and Pinter, Yuval},
 title={Attention is not not Explanation},
 booktitle={Proceedings of the 2019 Conference on Empirical Methods in Natural Language Processing and the 9th International Joint Conference on Natural Language Processing},
 year={2019}, pages={11--20}}

@inproceedings{bechler2024graphs,
  author = {Bechler-Speicher, Maya and Amos, Ido and Gilad-Bachrach, Ran and Globerson, Amir},
  title = {Graph Neural Networks Use Graphs When They Shouldn't},
  booktitle = {Proceedings of the 41st International Conference on Machine Learning},
  series = {Proceedings of Machine Learning Research},
  volume = {235},
  pages = {3284--3304},
  publisher = {PMLR},
  year = {2024},
  url = {https://proceedings.mlr.press/v235/bechler-speicher24a.html}
}

\appendix

\newcommand{\appendixcontentsentry}[2]{%
  \noindent\makebox[1.7em][l]{\textbf{\ref{#1}}}#2%
  \nobreak\enspace\dotfill\enspace\pageref{#1}\par}
\twocolumn[%
  \section*{Appendix}
  \begin{center}
  \begin{minipage}{0.72\textwidth}
    \noindent\textbf{Contents}\hfill\textbf{Page}\par\smallskip
    \appendixcontentsentry{app:selection}{Experimental Design and Model Selection}
    \appendixcontentsentry{app:frames}{Controlled Frame Experiments}
    \appendixcontentsentry{app:protocol}{Real-Graph Protocol}
    \appendixcontentsentry{app:families}{Control Families and Parameter Budgets}
    \appendixcontentsentry{app:procedures}{Interventions and Statistical Procedures}
    \appendixcontentsentry{app:implementation}{Implementation and Environment}
    \appendixcontentsentry{app:supp-results}{Supplementary Real-Graph Results}
    \appendixcontentsentry{app:mechanism}{Roman-Empire Mechanism Diagnostics}
    \appendixcontentsentry{app:newchecks}{Additional Tests of Map Learning and Edge Assignment}
    \appendixcontentsentry{app:isn}{Comparison with the Identity Sheaf Network}
  \end{minipage}
  \end{center}
]

\section{Experimental Design and Model Selection}\label{app:selection}

\subsection{Seeds and Run Counts}

Each frame spread uses 20 seeds ($0,\ldots,19$). In the main graph experiments, every map family uses optimization seed $k$ on official split $k$ for $k=0,\ldots,9$. The additional optimization-seed sensitivity check crosses splits $\{0,1,2\}$ with seeds $\{101,102,103\}$, yielding nine paired runs per dataset. We set the run seed for both CPU and GPU operations in PyTorch. For DSNN, we use the same seed for random draws in NumPy and Python.

The learning-rate search covers 1,680 training runs across the two architectures: $5\times10\times8\times4=1{,}600$ DNSD runs (datasets, splits, map families, and learning rates) and $2\times10\times4=80$ NSD runs on Roman-Empire and Minesweeper. Each architecture uses one run for each combination of dataset, split, map family and learning rate in its grid.

\subsection{Model Selection}

On each split, each DNSD family selects the learning rate in $\{.001,.003,.01,.03\}$ that maximizes validation score (Table~\ref{tab:lr-counts} lists Full and Resampled-F). At the selected rates we train Full and Resampled-F on matching splits and seeds, giving 50 pairs (100 fits). For the other families we use the fitted model selected by validation. Seven assignment summaries per checkpoint give 700 summaries in total.

Fixed-checkpoint interventions (native, identity, fixed-shuffle, and layer-mean) evaluate one Full checkpoint per split and dataset (50 checkpoints total), using the NSD reference learning rate of $.01$ on Roman-Empire and Minesweeper and the validation-selected rate on the other three datasets. Appendix~\ref{app:procedures} specifies the paired confidence intervals and sign-flip tests. For scores and score differences, plain numerical subscripts give standard deviations, and subscripts with $\pm$ give 95\% confidence interval half-widths.

\begin{table}[t]
\centering
\small
\setlength{\tabcolsep}{2.0pt}
\renewcommand{\arraystretch}{1.08}
\begin{tabular}{@{}p{.28\columnwidth}p{.25\columnwidth}p{.39\columnwidth}@{}}
\toprule
Dataset & Full & Resampled-F\\
\midrule
Roman-Empire & $.003{:}2,\ .01{:}8$ &
$.001{:}3,\ .003{:}4,$\newline $.01{:}1,\ .03{:}2$\\
\addlinespace[2pt]
Amazon-Ratings & $.001{:}6,\ .003{:}3,$\newline $.01{:}1$ &
$.001{:}4,\ .003{:}6$\\
\addlinespace[2pt]
Questions & $.001{:}5,\ .003{:}5$ & $.001{:}6,\ .003{:}4$\\
\addlinespace[2pt]
Tolokers & $.001{:}2,\ .003{:}3,$\newline $.01{:}3,\ .03{:}2$ &
$.001{:}4,\ .003{:}4,$\newline $.01{:}2$\\
\addlinespace[2pt]
Minesweeper & $.01{:}4,\ .03{:}6$ &
$.001{:}1,\ .01{:}4,$\newline $.03{:}5$\\
\bottomrule
\end{tabular}
\caption{Validation-selected learning rates across ten splits. Each learning rate is followed by the number of splits that selected it. Unlisted rates were never selected.}
\label{tab:lr-counts}
\end{table}

\section{Controlled Frame Experiments}\label{app:frames}

\subsection{Generation and Splits}

Messages are sampled from Eq.~\eqref{eq:frame-data} at angular spreads $\eta\in\{0,.25,.5,.75,1\}$ with $G=8$ frames, stalk dimension $d=6$, signal amplitude $s=2$, and noise standard deviation $\sigma=.75$. For each frame, we draw an orthogonal map acting on $e_1^\perp$ (fixing $e_1$) before applying the planar rotation. Each seed permutes balanced binary labels across 1,000 examples and splits them into 500 training, 250 validation, and 250 test instances.

\subsection{Calibration and Label-Only Fitting}

For Table~\ref{tab:calibration}, the map table is first fitted to the oracle alignment targets $(R_g^\top)_g$ for 200 steps, after which the linear readouts are fitted for 300 steps. For the label-only experiments, maps are initialized to the identity and optimized jointly with the linear readout, while the identity baseline fits its own linear readout on identical splits. Both use learning rate $.02$, weight decay $10^{-4}$, a 1,000-epoch cap, and early-stopping patience of 150 epochs.

\begin{table}[t]
\centering
\small
\setlength{\tabcolsep}{3.0pt}
\begin{tabular}{rrrrr}
\toprule
$\eta$ & Map error & Native & Post-I & Retrain-I\\
\midrule
$0$ & \resultstat{7.07{\times}10^{-5}}{6.3{\times}10^{-6}} & $1.000$ & $1.000$ & $1.000$\\
$.25$ & \resultstat{7.24{\times}10^{-5}}{5.4{\times}10^{-6}} & $1.000$ & $1.000$ & $1.000$\\
$.50$ & \resultstat{7.91{\times}10^{-5}}{5.2{\times}10^{-6}} & $1.000$ & $.833$ & $1.000$\\
$.75$ & \resultstat{8.38{\times}10^{-5}}{4.7{\times}10^{-6}} & $1.000$ & $.149$ & $.988$\\
$1$ & \resultstat{7.77{\times}10^{-5}}{5.7{\times}10^{-6}} & $1.000$ & $.501$ & $.492$\\
\bottomrule
\end{tabular}
\par\noindent\parbox{\columnwidth}{\footnotesize\textit{Note.} Subscripts give the sample SD of map error across 20 seeds.}
\caption{Correspondence-supervised calibration. Map error averages $\|T_g-R_g^\top\|_F$ over frames. The last three columns give mean test accuracy with fitted maps, identity maps inserted at test time, and a readout retrained on unaligned averages.}
\label{tab:calibration}
\end{table}

Table~\ref{tab:core-frame} gives absolute test accuracies, and Table~\ref{tab:label-only-full} reports paired accuracy differences alongside two fitted-map diagnostics. Using the readout-orthogonal component $B_g$ in Eq.~\eqref{eq:nullmaps} and the Frobenius norm $\|\cdot\|_F$, the null-space energy fraction is
\[
E_{\rm null}=\frac{\sum_g\|B_g\|_F^2}{\sum_g\|T_g\|_F^2},
\]
compared with the dimension share $(d^2-d)/d^2=5/6=.833$. Mean distance from identity is $M_F=G^{-1}\sum_g\|T_g-I_d\|_F$.

\begin{table}[t]
\centering
\small
\setlength{\tabcolsep}{3.0pt}
\begin{tabular}{rrrrr}
\toprule
$\eta$ & $R_I$ & $N_I$ & $E_{\rm null}$ & $M_F$\\
\midrule
$0$ & \resultstat{.225}{.302} & \resultstat{-.001}{.003} & \resultstat{.788}{.046} & \resultstat{.725}{.393}\\
$.25$ & \resultstat{.214}{.252} & \resultstat{-.001}{.004} & \resultstat{.792}{.037} & \resultstat{.676}{.297}\\
$.50$ & \resultstat{.342}{.231} & \resultstat{.001}{.006} & \resultstat{.788}{.046} & \resultstat{.775}{.266}\\
$.75$ & \resultstat{.356}{.173} & \resultstat{.015}{.009} & \resultstat{.784}{.035} & \resultstat{.858}{.178}\\
$1$ & \resultstat{.501}{.030} & \resultstat{.507}{.033} & \resultstat{.769}{.052} & \resultstat{.948}{.169}\\
\bottomrule
\end{tabular}
\par\noindent\parbox{\columnwidth}{\footnotesize\textit{Note.} Subscripts give sample SDs across 20 seeds. $R_I$ and $N_I$ are Full minus identity accuracy with the readout fixed and retrained, respectively.}
\caption{Identity contrasts and map diagnostics in the label-only experiment.}
\label{tab:label-only-full}
\end{table}

\section{Real-Graph Protocol}\label{app:protocol}

\subsection{Graph Inputs and Evaluation Metrics}

We use the supplied node features and edges directly. Table~\ref{tab:datasets} counts each undirected edge twice, once in each direction. On Roman-Empire, edges represent word order and syntactic dependencies in a Wikipedia article, and on Questions they capture interactions among Yandex Q users.

\begin{table*}[t]
\centering
\small
\setlength{\tabcolsep}{5.0pt}
\begin{tabular}{lrrrl}
\toprule
Dataset & Nodes & Directed entries & Features per node & Metric\\
\midrule
Roman-Empire & $22{,}662$ & $65{,}854$ & $300$ & Accuracy\\
Amazon-Ratings & $24{,}492$ & $186{,}100$ & $300$ & Accuracy\\
Questions & $48{,}921$ & $307{,}080$ & $301$ & AUROC\\
Tolokers & $11{,}758$ & $1{,}038{,}000$ & $10$ & AUROC\\
Minesweeper & $10{,}000$ & $78{,}804$ & $7$ & AUROC\\
\bottomrule
\end{tabular}
\caption{Graph sizes and evaluation metrics for the node-classification datasets of \citet{platonov2023critical}.}
\label{tab:datasets}
\end{table*}

\subsection{Optimization}

We train DNSD and NSD with Adam for up to 500 epochs, using weight decay $5\times10^{-4}$ and clipping the gradient norm at 1.0. We halve the learning rate when validation performance plateaus (patience 20 epochs) and stop training after 100 epochs without improvement. The NSD benchmark matches DNSD in depth, hidden width, and stalk dimension. Deterministic models run a single forward pass at evaluation.

On Texas, DSNN follows the published configuration of \citet{fiorini2026reloaded}: stalk dimension $d=3$, four layers, hidden width 20, learning rate $.02$, weight decay $5\times10^{-3}$, dropout $.7$, a 1,500-epoch cap, and early-stopping patience of 200 epochs. Resampled averages eight predicted probability tensors at evaluation, while Native and Fixed use a single pass.

\section{Control Families and Parameter Budgets}\label{app:families}

\subsection{Restriction Map Predictors}

Each DNSD layer uses separate source and target map predictors, each mapping concatenated endpoint node embeddings in $\mathbb R^{2h}$ to $d^2$ matrix entries with one bias per entry. Across $L$ layers, the full-matrix predictors therefore contain
\begin{equation}
P_{\rm Full}=2L(2h+1)d^2,
\end{equation}
where $h$ is the hidden width and $d$ is the stalk dimension. Reassigning predicted pairs in Fixed-F or Resampled-F leaves this parameter count unchanged. Restricting predictors to diagonal matrices replaces $d^2$ with $d$, yielding $16{,}512$ parameters for the eight-layer backbone, while sharing one source--target pair per layer uses $2Ld^2$ parameters.

\subsection{Capacity-F and Shared-C}

Each residual adapter acts on $x\in\mathbb R^h$ as
\begin{equation}
x\mapsto x+W_2\operatorname{SiLU}(W_1x+b_1)+b_2,
\end{equation}
where $W_1$ maps an $h$-dimensional node representation to $m$ hidden channels, $W_2$ projects it back to $h$, and $b_1,b_2$ are biases. This gives $(2h+1)m+h$ parameters per adapter. With $m=127$ for Capacity-F and $m=126$ for Shared-C, the controlled parameters total $99.61\%$ and $99.60\%$ of Full's map budget, respectively (Table~\ref{tab:budgets}).

\begin{table}[t]
\centering
\small
\setlength{\tabcolsep}{3.8pt}
\begin{tabular}{lrrr}
\toprule
Family & Map params. & Adapter params. & Controlled total\\
\midrule
Full & $132{,}096$ & $0$ & $132{,}096$\\
Capacity-F & $0$ & $131{,}576$ & $131{,}576$\\
Shared-F & $1{,}024$ & $0$ & $1{,}024$\\
Shared-C & $1{,}024$ & $130{,}544$ & $131{,}568$\\
\bottomrule
\end{tabular}
\caption{Parameters allocated to restriction maps and node-wise adapters in DNSD. All variants use the same input, propagation, and output modules within a dataset.}
\label{tab:budgets}
\end{table}

\section{Interventions and Statistical Procedures}\label{app:procedures}

\subsection{Fixed-Checkpoint Interventions}

All three interventions act on each layer's restriction-map pairs immediately before propagation. Fixed shuffling permutes rows of the source--target pair table (one row per directed message) using a layer-specific permutation $\pi_\ell$. For the layer-mean intervention, the source and target restriction matrices are averaged separately across edges before forming their transport product in Eq.~\eqref{eq:action}, which differs from averaging edge transport products directly.

\subsection{Retrained Assignment Controls}

In Fixed-F, one permutation is sampled per seed and layer and reused across training, validation, and testing. In Resampled-F, a fresh permutation is drawn at every forward pass. Whenever one edge receives another edge's restriction-map pair, gradients flow back through that pair to the predictor output that generated it.

For DNSD we score Resampled-F on validation and test nodes after averaging logits over eight independent assignments. To measure the effect of averaging at inference (Table~\ref{tab:resampling}), we draw 16 assignments per checkpoint. We compare the mean of their individual scores with the score obtained after averaging the first $K\in\{1,2,4,8,16\}$ predicted probability tensors.

\subsection{Paired Summaries}

For each dataset and control, let $d_i=S_i(\text{Full})-S_i(\text{control})$
be the difference on official split $i\in\{1,\ldots,10\}$. The reported mean
and paired 95\% interval are
\begin{equation}
\bar d=\frac{1}{10}\sum_{i=1}^{10}d_i,
\qquad
\bar d\ \pm\ t_{.975,9}\frac{s_d}{\sqrt{10}},
\end{equation}
where $\bar d$ and $s_d$ are the sample mean and standard deviation of the ten paired differences and $t_{.975,9}$ is the 97.5th percentile of Student's $t$ distribution with nine degrees of freedom.

For Full versus Resampled-F, we also use an exact two-sided paired sign-flip test over all $2^{10}$ sign assignments, with Holm adjustment across the five datasets \citep{holm1979simple}. The nine-run optimization-seed interval treats each run as one unit (runs sharing a split reuse its test nodes), and the DSNN interval pairs the ten official splits.

\section{Implementation and Environment}\label{app:implementation}

With the published diagonal DNSD configuration, our five-seed mean is \resultstat{.811}{.004} accuracy on Roman-Empire and \resultstat{.889}{.006} AUROC on Minesweeper. \citet{bourgerie2026deep} report accuracy on both datasets, \resultstat{.834}{.009} and \resultstat{.894}{.008}, respectively. We use AUROC throughout our Minesweeper comparisons and the same implementation across map families, training budgets and assignment controls.

The DNSD and NSD benchmark runs used Ubuntu 20.04, Python 3.11.15, PyTorch 2.7.1, CUDA 12.8, PyTorch Geometric 2.6.1 and one NVIDIA RTX 4090 GPU. We checked that restoring the fitted weights reproduces the mean test logits within $10^{-5}$. Substituting identity maps gives the same predictions as identity propagation, and shuffling keeps each source--target map pair together.

\section{Supplementary Real-Graph Results}\label{app:supp-results}

\subsection{Fixed-Checkpoint Interventions}

Identity, shuffle, and layer-mean replacement each lower the score at all 50 Full checkpoints (smallest individual drops: $.0975$ for identity and $.0042$ for shuffle). Table~\ref{tab:interventions} reports drop intervals and layer-mean prediction changes. For the separately fitted Fixed-F models in Table~\ref{tab:newdiag}, we subtract the mean test-time shuffle score from the score under the training permutation.

\begin{table*}[!htbp]
\centering\small
\setlength{\tabcolsep}{4pt}
\begin{tabular}{@{}lccccc@{}}
\toprule
 & \multicolumn{3}{c}{Mean score drop} & \multicolumn{2}{c}{Layer mean}\\
\cmidrule(lr){2-4}\cmidrule(l){5-6}
Dataset & Identity & Shuffle & Fixed-F shuffle & Flips & Logit change\\
\midrule
Roman-Empire & \resultstat{.3231}{\pm .0169} & \resultstat{.2599}{\pm .0087} & \resultstat{.0006}{\pm .0016} & .373 & 10.307\\
Amazon-Ratings & \resultstat{.2371}{\pm .0093} & \resultstat{.1255}{\pm .0326} & \resultstat{.0063}{\pm .0039} & .542 & 3.164\\
Questions & \resultstat{.1223}{\pm .0158} & \resultstat{.0254}{\pm .0089} & \resultstat{.0011}{\pm .0020} & .003 & .374\\
Tolokers & \resultstat{.1940}{\pm .0202} & \resultstat{.0943}{\pm .0318} & \resultstat{.0009}{\pm .0016} & .132 & .996\\
Minesweeper & \resultstat{.3714}{\pm .0753} & \resultstat{.1101}{\pm .0370} & \resultstat{.0006}{\pm .0016} & .138 & 1.803\\
\bottomrule
\end{tabular}
\par\noindent\parbox{\textwidth}{\footnotesize\textit{Note.} Subscripts give paired 95\% CI half-widths across ten splits. Fixed-F half-widths are approximate. Flips gives the fraction of predictions that change. Logit change is the mean $\ell_2$ distance between native and intervened logit vectors.}
\caption{Score drops under identity replacement and shuffling in Full (Table~\ref{tab:core-audit}), test-time reshuffling in Fixed-F (Table~\ref{tab:newdiag}), and prediction changes under layer-mean replacement in Full.}
\label{tab:interventions}
\end{table*}

\subsection{Paired Differences After Retraining}\label{app:graphdetail}

\begin{table*}[t]
\centering\small
\setlength{\tabcolsep}{2pt}
\begin{tabular*}{\textwidth}{@{\extracolsep{\fill}}lccccc@{}}
\toprule
& \multicolumn{2}{c}{Accuracy} & \multicolumn{3}{c}{AUROC}\\
\cmidrule(lr){2-3}\cmidrule(l){4-6}
Control & Roman-Empire & Amazon-Ratings & Questions & Tolokers & Minesweeper\\
\midrule
Identity & \resultstat{+.0702}{\pm .0048} & \resultstat{+.0470}{\pm .0071} & \resultstat{+.0425}{\pm .0086} & \resultstat{+.0058}{\pm .0060} & \resultstat{-.0102}{\pm .0064}\\
Diagonal & \resultstat{+.0287}{\pm .0039} & \resultstat{+.0173}{\pm .0061} & \resultstat{+.0195}{\pm .0076} & \resultstat{+.0060}{\pm .0048} & \resultstat{-.0036}{\pm .0077}\\
Capacity-F & \resultstat{+.0496}{\pm .0033} & \resultstat{+.0563}{\pm .0086} & \resultstat{+.0328}{\pm .0081} & \resultstat{+.0074}{\pm .0058} & \resultstat{-.0102}{\pm .0073}\\
Shared-F & \resultstat{+.0655}{\pm .0038} & \resultstat{+.0326}{\pm .0071} & \resultstat{+.0426}{\pm .0065} & \resultstat{+.0061}{\pm .0041} & \resultstat{-.0135}{\pm .0061}\\
Shared-C & \resultstat{+.0391}{\pm .0040} & \resultstat{+.0428}{\pm .0065} & \resultstat{+.0310}{\pm .0065} & \resultstat{+.0043}{\pm .0041} & \resultstat{-.0118}{\pm .0064}\\
Fixed-F & \resultstat{+.1253}{\pm .0103} & \resultstat{+.0241}{\pm .0075} & \resultstat{-.0041}{\pm .0056} & \resultstat{+.0022}{\pm .0045} & \resultstat{+.0301}{\pm .0063}\\
Resampled-F & \resultstat{+.0675}{\pm .0036} & \resultstat{+.0024}{\pm .0052} & \resultstat{-.0025}{\pm .0038} & \resultstat{-.0027}{\pm .0035} & \resultstat{+.0089}{\pm .0120}\\
\bottomrule
\end{tabular*}
\par\noindent\parbox{\textwidth}{\footnotesize\textit{Note.} Positive differences favor Full. Subscripts give paired 95\% CI half-widths across ten official splits.}
\caption{Mean score differences between Full and each retrained control. Each family uses its validation-selected learning rate (Appendix~\ref{app:selection}).}
\label{tab:paired}
\end{table*}

On Roman-Empire, the ten split-wise Full-minus-Resampled-F differences range from $+.0609$ to $+.0768$, and all 20 selected fits stop before epoch 500. The exact sign-flip test attains its smallest two-sided $p$-value of $2/1024$ ($.00195$), or $.009766$ after Holm adjustment across five tasks.

To bound Full's gain over the strongest of seven controls, we construct simultaneous two-sided Bonferroni intervals for all seven paired differences and take the smallest lower and upper endpoints. On Roman-Empire, Full's mean accuracy gain over the strongest control is \resultstat{.02866}{\pm .00604}. On Minesweeper, Shared-F's mean AUROC gain over Full is \resultstat{.01349}{\pm .00933}. The half-widths account for all seven comparisons at 95\% simultaneous confidence. Table~\ref{tab:paired} gives the individual paired comparisons.

\subsection{Sensitivity to Learning Rate and Optimization Seed}

At each of the four common learning rates, Full outperforms Resampled-F on Roman-Empire before family-specific selection (Figure~\ref{fig:lr}). Across three additional optimization seeds on each of three splits, all nine paired differences on Roman-Empire also favor Full. On Questions, the nine-run mean favors Resampled-F (Table~\ref{tab:seed-check}).

\begin{figure*}[!t]
\centering
\includegraphics[width=\textwidth]{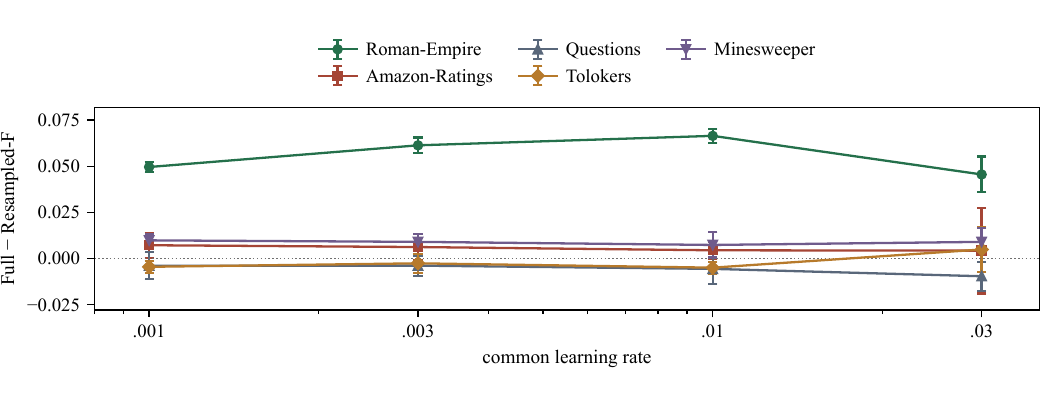}
\par\noindent\parbox{\textwidth}{\footnotesize\textit{Note.} Points give mean paired score differences across ten official splits. Error bars show paired 95\% CIs.}
\caption{Full minus Resampled-F at each common learning rate.}
\label{fig:lr}
\end{figure*}

\begin{table}[!htbp]
\centering
\small
\setlength{\tabcolsep}{2.25pt}
\begin{tabular}{@{}lrrr@{}}
\toprule
Dataset & Full & Resampled-F & Difference\\
\midrule
Roman-Empire & $.8383$ & $.7762$ & \resultstat{+.0621}{\pm .0053}\\
Questions & $.7747$ & $.7808$ & \resultstat{-.0061}{\pm .0059}\\
\bottomrule
\end{tabular}
\par\noindent\parbox{\columnwidth}{\footnotesize\textit{Note.} Subscripts give descriptive paired 95\% CI half-widths. Each split--seed pair counts as one observation. Runs on the same split share test nodes.}
\caption{Full minus Resampled-F over nine paired runs per dataset.}
\label{tab:seed-check}
\end{table}

\subsection{Assignment at Inference and Operator Transfer}

For Resampled-F, averaging probabilities across 16 assignments gives its largest mean change on Roman-Empire ($.0024$ relative to the mean single-assignment score) (Table~\ref{tab:resampling}). Randomizing a native Full checkpoint has a larger effect, even with 16-draw probability averaging. Table~\ref{tab:dsnn} gives the split variation for the operator-transfer experiment on Texas.

\begin{table}[!htbp]
\centering
\small
\setlength{\tabcolsep}{3.0pt}
\begin{tabular}{lrrr}
\toprule
\multicolumn{4}{c}{Full checkpoint}\\
\cmidrule(lr){1-4}
Dataset & Native & $\mathbb E[S_1]$ & Avg. $K=16$\\
\midrule
Roman-Empire & \resultstat{.8445}{.0049} & \resultstat{.5568}{.0612} & \resultstat{.6148}{.0354}\\
Amazon-Ratings & \resultstat{.4781}{.0089} & \resultstat{.3587}{.0340} & \resultstat{.3698}{.0401}\\
Questions & \resultstat{.7720}{.0101} & \resultstat{.7473}{.0163} & \resultstat{.7506}{.0153}\\
Tolokers & \resultstat{.8464}{.0060} & \resultstat{.7507}{.0593} & \resultstat{.7556}{.0591}\\
Minesweeper & \resultstat{.9347}{.0102} & \resultstat{.8511}{.0303} & \resultstat{.8798}{.0205}\\
\addlinespace
\multicolumn{4}{c}{Resampled-F checkpoint}\\
\cmidrule(lr){1-4}
Dataset & Avg. $K=1$ & $\mathbb E[S_1]$ & Avg. $K=16$\\
\midrule
Roman-Empire & \resultstat{.7743}{.0069} & \resultstat{.7749}{.0067} & \resultstat{.7773}{.0069}\\
Amazon-Ratings & \resultstat{.4738}{.0085} & \resultstat{.4746}{.0073} & \resultstat{.4755}{.0072}\\
Questions & \resultstat{.7746}{.0089} & \resultstat{.7742}{.0088} & \resultstat{.7748}{.0089}\\
Tolokers & \resultstat{.8488}{.0081} & \resultstat{.8488}{.0082} & \resultstat{.8492}{.0082}\\
Minesweeper & \resultstat{.9250}{.0128} & \resultstat{.9253}{.0128} & \resultstat{.9259}{.0128}\\
\bottomrule
\end{tabular}
\par\noindent\parbox{\columnwidth}{\footnotesize\textit{Note.} Subscripts give sample SDs across ten splits.}
\caption{Test scores under random reassignment. $\mathbb E[S_1]$ averages scores from 16 individual assignments. Avg.\ $K$ scores the mean predicted probabilities over the first $K$ draws.}
\label{tab:resampling}
\end{table}

\begin{table}[!htbp]
\small
\centering\setlength{\tabcolsep}{4.0pt}
\begin{tabular}{lr}
\toprule
Assignment & Accuracy\\
\midrule
Native & \resultstat{.8000}{.0386}\\
Fixed & \resultstat{.8000}{.0528}\\
Resampled & \resultstat{.8135}{.0517}\\
\midrule
Native $-$ Resampled & \resultstat{-.0135}{\pm .0188}\\
\bottomrule
\end{tabular}
\par\noindent\parbox{\columnwidth}{\footnotesize\textit{Note.} Plain subscripts give sample SDs across ten splits. For the paired difference, $\pm$ gives the 95\% CI half-width.}
\caption{DSNN accuracy under three edge assignments on Texas.}
\label{tab:dsnn}
\end{table}

\section{Roman-Empire Mechanism Diagnostics}\label{app:mechanism}

For each layer of a Full checkpoint, let $a_e$ vectorize the entries of the directed edge transport map $A_e$ from Eq.~\eqref{eq:action}, and let $\bar a$ denote its mean across edges. The normalized dispersion is
\begin{equation}
D=\sqrt{\frac{\sum_e\|a_e-\bar a\|_2^2}{\sum_e\|a_e\|_2^2}}.
\end{equation}
The effective rank of the covariance of $a_e$ is the exponential of the Shannon entropy of its normalized eigenvalues. On a fixed sample of up to 20,000 edges per layer, we correlate $\|a_e-\bar a\|_2$ with endpoint feature cosine similarity and $\log(1+\deg(u)+\deg(v))$. The Roman-Empire checkpoints in Figure~\ref{fig:layers} give $D=.865$ and a feature-cosine correlation of $-.034$. Table~\ref{tab:newdiag} reports the same diagnostics for an independent set of Full fits across all five benchmarks. Across depth, effective rank rises steadily, while dispersion dips at layer two before rising through layer eight (Figure~\ref{fig:layers}).

\begin{figure*}[!t]
\centering
\includegraphics[width=\textwidth]{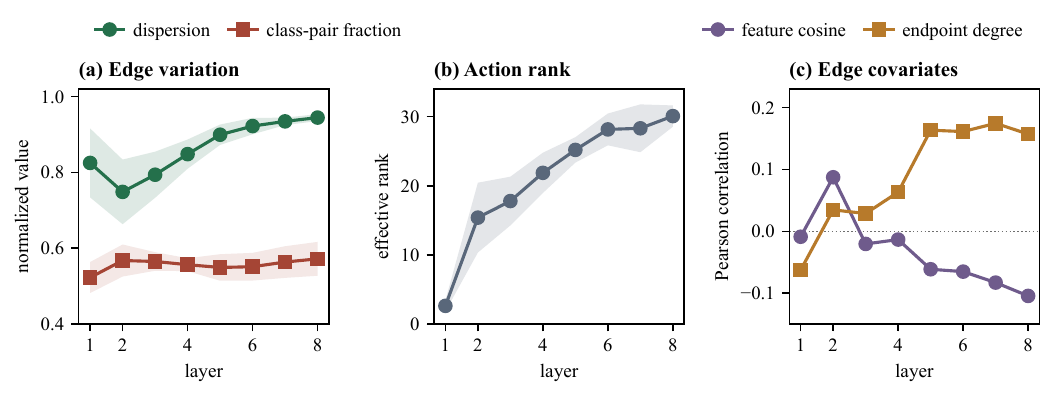}
\par\noindent\parbox{\textwidth}{\footnotesize\textit{Note.} Curves give means across ten official splits. Shading shows one sample SD on either side of the mean.}
\caption{Roman-Empire edge transport diagnostics across DNSD layers.}
\label{fig:layers}
\end{figure*}

After model selection, endpoint class labels define the groups used to compute the explained variance fraction $H$ in Eq.~\eqref{eq:classpair} and the same-label versus different-label strata in Table~\ref{tab:strata}, while the frozen network receives its standard node features. Because each intervention replaces restriction maps across an entire stratum, the resulting accuracy drop reflects the aggregate effect over that edge count.

\begin{table}[!htbp]
\centering
\small
\setlength{\tabcolsep}{4.0pt}
\begin{tabular}{lrr}
\toprule
Stratum & Directed edges & Accuracy drop\\
\midrule
Same label & $3{,}088$ & \resultstat{.0112}{.0023}\\
Different label & $62{,}766$ & \resultstat{.2549}{.0604}\\
Low feature cosine & $32{,}928$ & \resultstat{.1325}{.0123}\\
High feature cosine & $32{,}926$ & \resultstat{.1253}{.0088}\\
\bottomrule
\end{tabular}
\par\noindent\parbox{\columnwidth}{\footnotesize\textit{Note.} Subscripts give sample SDs across ten splits.}
\caption{Accuracy loss when source and target maps are replaced by their respective layer means on the indicated edges. The low and high groups split edges at the median endpoint feature cosine.}
\label{tab:strata}
\end{table}

\section{Additional Tests of Map Learning and Edge Assignment}
\label{app:newchecks}

\subsection{Changing the initial maps}
At angular spread $\eta=0$, we repeat the label-only frame fits from Appendix~\ref{app:frames} under identity, Gaussian, and Haar-random orthogonal map initializations. Using each run's final readout weights in Eq.~\eqref{eq:nullmaps}, we compute the null-space squared-norm fraction for both the fitted maps $T_g^{\rm final}$ and their cumulative updates $T_g^{\rm final}-T_g^{\rm init}$.

\begin{table}[ht]
\centering\small
\begin{tabular}{lrrrcc}
\toprule
Start & Full & Post-I & Retrain-I & Final & Update\\
\midrule
Identity & .998 & .773 & .999 & .788 & .182\\
Gaussian & .996 & .401 & .999 & .763 & .289\\
Haar & .998 & .480 & .999 & .776 & .347\\
\bottomrule
\end{tabular}
\caption{Test accuracy with learned maps (Full), identity maps substituted at the fixed checkpoint (Post-I), and a readout retrained with identity maps (Retrain-I). Final and Update report the corresponding null-space energy fractions for Full.}
\label{tab:newinit}
\end{table}

\subsection{Nonlinear readouts}
Across ten seeds each at $\eta=0$ and $\eta=1$, we differentiate the scalar logit with respect to one $6\times6$ message map with all other parameters frozen. Evaluated on 500 probe inputs, this forms a $500\times36$ Jacobian whose numerical rank in double precision (relative singular-value threshold $10^{-8}$) defines local nullity as $36$ minus that rank (Table~\ref{tab:newnull}).

For a nonlinear readout $h(PTx)$ whose first layer $P$ projects into $k$ hidden units with rank $k$, any perturbation $B$ with $PB=0$ leaves $PTx$ unchanged for every $x$. These perturbations form a $d(d-k)$-dimensional subspace for an unrestricted $d\times d$ map.

The width sweep uses a single-hidden-layer readout $a^\top\tanh(Pz)+b$ applied to the mean-aggregated representation $z=G^{-1}\sum_gT_gX_g$, with one row of $P$ per hidden unit. At each hidden width, all 20 checkpoints yield identical local nullity (verified identically with 1,000 probe inputs). Replacing $\tanh$ with a post-map ReLU over 500 inputs produces zero local nullity on 17 of 20 checkpoints, while the remaining three have nullities of 4, 15, and 20.

\begin{table}[ht]
\centering\small
\begin{tabular}{@{}lr@{}}
\toprule
One-map setting & Local nullity\\
\midrule
Linear or channel gate & 30\\
Layer normalization & 6\\
$\tanh$ alone or combined$^{*}$ & 0\\
\midrule
Fixed $\tanh$ readout width & Local nullity\\
$1,\ 2,\ 3,\ 4,\ 5,\ 6$ & $30,\ 24,\ 18,\ 12,\ 6,\ 0$\\
\bottomrule
\end{tabular}
\caption{Local nullity for component ablations and the readout width sweep at $d=6$. $^{*}$Combined uses $\tanh$, a fixed channel gate, and layer normalization.}
\label{tab:newnull}
\end{table}

\subsection{Training budget and assignment noise}
On Roman-Empire, expanding the learning-rate grid to six values spanning $.001$ to $.1$ gives Fixed-F $.7417$ accuracy and Resampled-F $.7797$ under a 500-epoch cap, versus $.7358$ and $.7795$ when the cap is raised to 1,000 epochs. At 1,000 epochs the mean accuracy gain of Full over Resampled-F is \resultstat{.0636}{\pm .0042}. Subscripts here give approximate paired 95\% confidence interval half-widths across the ten splits.

At fixed parameter values, averaging gradients across eight independent map assignments per step (R8) reduces raw-gradient variance relative to a single assignment (R1) to a median ratio of $.123$. Both protocols average eight validation draws for learning-rate selection, and the fixed-budget setting evaluates the final checkpoint. On DNSD's neighbor branch, the mean R8 minus R1 accuracy difference is \resultstat{.0026}{\pm .0090} at 500 epochs and \resultstat{.0004}{\pm .0037} with learning-rate scheduling and early stopping. Full's advantage over R8 persists after reducing gradient variance (Table~\ref{tab:newbranches}).

\subsection{Dispersion and test-time reassignment}
Sensitivity to test-time shuffling also reflects training assignment. On every benchmark, Fixed-F suffers a smaller mean drop when reshuffled at inference than Full does, with four of the five paired 95\% intervals for the Fixed-F drop spanning zero and the Amazon-Ratings interval above zero (Table~\ref{tab:newdiag} and Table~\ref{tab:interventions}). On Roman-Empire, Fixed-F trails Full under its training permutation and loses almost no further accuracy when reassigned again at inference.

\begin{table}[t]
\centering\small
\begin{tabular}{lrrr}
\toprule
Dataset & $D$ & $r$ & Fixed-F drop\\
\midrule
Roman-Empire & .869 & $-.041$ & .0006\\
Amazon-Ratings & .572 & $-.092$ & .0063\\
Questions & .252 & $-.336$ & .0011\\
Tolokers & .460 & $-.167$ & .0009\\
Minesweeper & .551 & $-.144$ & .0006\\
\bottomrule
\end{tabular}
\caption{Normalized dispersion $D$ of Full's edge transport maps across edges in independently retrained fits, Pearson correlation $r$ between transport deviation $\|A_e-\bar A\|_F$ and endpoint feature cosine similarity, and mean test-time reshuffling loss in separately trained Fixed-F models across ten splits (paired 95\% CIs in Table~\ref{tab:interventions}).}
\label{tab:newdiag}\end{table}

Comparing Questions and Minesweeper shows that fitted map dispersion and the retraining cost of sharing rank in opposite orders. In independently retrained Full fits, normalized dispersion $D$ measures the root-mean-square deviation of edge transport maps from their layer mean divided by their root-mean-square norm (Appendix~\ref{app:mechanism}, Table~\ref{tab:newdiag}). Questions has the lowest dispersion ($.252$) and a layer-mean checkpoint drop of $.019$ (Table~\ref{tab:core-audit}), yet retraining with one shared pair lowers AUROC by $.0426$. Minesweeper has over twice the dispersion ($.551$), yet sharing raises retrained AUROC by $.0135$ (Table~\ref{tab:paired}). The graph with less fitted map variation thus pays the larger sharing penalty.

\subsection{Comparison across DNSD propagation branches}
\begin{table}[t]
\centering\small
\begin{tabular}{@{}lcc@{}}
\toprule
Branch & Fixed 500 epochs & Early stopping\\
\midrule
Neighbor & \resultstat{.0782}{\pm .0064} & \resultstat{.0652}{\pm .0031}\\
Coboundary & \resultstat{.0741}{\pm .0087} & \resultstat{.0540}{\pm .0088}\\
Difference & \resultstat{.0041}{\pm .0109} & \resultstat{.0112}{\pm .0084}\\
\bottomrule
\end{tabular}
\par\noindent\parbox{\columnwidth}{\footnotesize\textit{Note.} Difference is neighbor minus coboundary. Subscripts give approximate paired 95\% CI half-widths over ten splits.}
\caption{Roman-Empire accuracy gain of Full over R8 (eight assignment gradients per update) by DNSD branch.}
\label{tab:newbranches}\end{table}

We compared both DNSD propagation branches \citep{bourgerie2026deep}, contrasting direct neighbor aggregation against sender--receiver coboundary differences. Full outperforms R8 on both branches (Table~\ref{tab:newbranches}). Under early stopping the neighbor-branch margin is larger, while under 500 epochs the branch-difference interval spans zero.

\section{Comparison with the Identity Sheaf Network}
\label{app:isn}

\citet{hernandez2026necessity} compare the Identity Sheaf Network (ISN) against published SNN results on WebKB (Texas, Wisconsin, and Cornell) \citep{craven1998webkb,pei2020geom} and Wikipedia (Squirrel and Chameleon) \citep{rozemberczki2021multiscale}, relating accuracy to neighborhood class composition and tracking layer-wise Rayleigh quotients to diagnose oversmoothing. Table~\ref{tab:isn-scope} contrasts the measurements of ISN and LRN-Evaluation, where we compare DNSD models on paired splits of the five benchmarks in Table~\ref{tab:main}.

\begin{table}[htbp]
\centering
\small
\setlength{\tabcolsep}{3pt}
\begin{tabular}{@{}p{.37\columnwidth}cp{.47\columnwidth}@{}}
\toprule
Measurement & ISN & This work\\
\midrule
\raggedright Layer-wise Rayleigh quotients & Yes & ---\tabularnewline
\raggedright Change during training & --- & \raggedright Map predictor weights since initialization\tabularnewline
\raggedright Train with identity maps & Yes & \raggedright Paired Full and Identity fits\tabularnewline
\raggedright Intervene at a checkpoint & --- & \raggedright Test score with identity, shuffled, or averaged maps\tabularnewline
\raggedright Match parameter budgets & --- & \raggedright Node-wise adapters with identity or shared maps\tabularnewline
\raggedright Reassign maps during training & --- & \raggedright Fixed or freshly sampled permutation\tabularnewline
\raggedright Reliance and replacement & --- & \raggedright Exact accuracies with fixed and refitted readouts\tabularnewline
\bottomrule
\end{tabular}
\caption{Measurements in ISN and LRN-Evaluation. A dash means the study did not report that measurement. Exact accuracies refer to the Gaussian model in Theorem~\ref{thm:frame}.}
\label{tab:isn-scope}
\end{table}

\end{document}